\documentclass[a4paper]{article}
\usepackage[all]{xy}\usepackage[latin1]{inputenc}        
\usepackage[dvips]{graphics,graphicx}
\usepackage{amsfonts,amssymb,amsmath,color,mathrsfs, amstext}
\usepackage{amsbsy, amsopn, amscd, amsxtra, amsthm,authblk}
\usepackage{enumerate,algorithmic,algorithm}
\usepackage{fancyhdr}

\usepackage{upref,ulem}
\usepackage{geometry}
\usepackage{float}
\usepackage{subfigure}
\usepackage{yhmath}
\usepackage{booktabs}
\usepackage{graphicx}
\usepackage{multirow}
\usepackage{makecell}
\usepackage[colorlinks,
            linkcolor=red,
            anchorcolor=red,
            citecolor=red
            ]{hyperref}

\numberwithin{equation}{section}

\newcommand{\R}{\mathbb{R}}

\newcommand \footnoteONLYtext[1]
{
	\let \mybackup \thefootnote
	\let \thefootnote \relax
	\footnotetext{#1}
	\let \thefootnote \mybackup
	\let \mybackup \imareallyundefinedcommand
}

\newtheorem{theorem}{Theorem}[section]

\newtheorem{definition}{Definition}[section]
\newtheorem{remark}{Remark}[section]
\newtheorem{proposition}{Proposition}[section]

\numberwithin{equation}{section}

\begin{document}

\title{A deep learning framework for geodesics under spherical Wasserstein-Fisher-Rao metric and its application for weighted sample generation}

\author[1]{Yang Jing \thanks{sharkjingyang@sjtu.edu.cn} }
\author[2]{Jiaheng Chen \thanks{sjtuchenjiaheng@sjtu.edu.cn}  }
\author[3]{Lei Li \thanks{leili2010@sjtu.edu.cn}}
\author[4]{Jianfeng Lu \thanks{jianfeng@math.duke.edu}}
\affil[1,2,3]{School of Mathematical Sciences, Shanghai Jiao Tong University, Shanghai, 200240, P.R.China.}
\affil[3]{Institute of Natural Sciences, Qing Yuan Research Institute, MOE-LSC, Shanghai Jiao Tong University, Shanghai, 200240, P.R.China.}
\affil[4]{Mathematics Department, Duke University, Box 90320, Durham NC, 27708, USA.}

\date{}
\maketitle

\begin{abstract}
Wasserstein-Fisher-Rao  (WFR) distance is a family of metrics to gauge the discrepancy of two Radon measures, which takes into account both transportation and weight change. Spherical WFR distance is a projected version of WFR distance for probability measures so that the space of Radon measures equipped with WFR can be viewed as metric cone over the space of probability measures with spherical WFR. Compared to the case for Wasserstein distance, the understanding of geodesics under the spherical WFR is less clear and still an ongoing research focus. In this paper, we develop a deep learning framework to compute the geodesics under the spherical WFR metric, and the learned geodesics can be adopted to generate weighted samples. Our approach is based on a Benamou-Brenier type dynamic formulation for spherical WFR. To overcome the difficulty in enforcing the boundary constraint brought by the weight change, a Kullback-Leibler (KL) divergence term based on the inverse map is introduced into the cost function. Moreover, a new regularization term using the particle velocity is introduced as a substitute for the Hamilton-Jacobi equation for the potential in dynamic formula. When used for sample generation, our framework can be beneficial for applications with given weighted samples, especially in the Bayesian inference, compared to sample generation with previous flow models.  
\end{abstract}

\section{Introduction}
\footnoteONLYtext{All authors contributed equally.}
A proper metric of probability measures can lead to new methods and algorithms in data science, like the SVGD \cite{liu2016stein,liu2017stein} and WGAN \cite{arjovsky2017wasserstein,gulrajani2017improved} methods in unsupervised learning. The family of Wasserstein distances is among the popular metrics since they can evaluate dissimilarity between two distributions, even one or both of them are discrete data samples with disjoint supports. Wasserstein distance have quantity of state-of-art properties, which can be applied to traditional models for improvements \cite{onken2020ot,salimans2018improving}.  Wasserstein-Fisher-Rao (WFR) distance (a.k.a, Hellinger-Kantorovich distance) is a generalized version of Wasserstein distance \cite{liero2018optimal, chizat2018interpolating,chizat2018unbalanced}, which interpolates between the quadratic Wasserstein and the Fisher-Rao metrics and generalizes optimal transport to measures with different masses. From  dynamic view, WFR adds a source term in its Benamou-Brenier formula, which will lead to a weighted particle transport process compared with Wasserstein distance. The spherical WFR metric is a modified version of WFR metric for probability measures and the corresponding particle motions also have weight change besides transportation \cite{laschos2019geometric,kondratyev2019spherical}.

Making use of the metric, one may transform a probability measure continuously into another one, by, for example, the gradient flows with certain carefully designed functional and the geodesics \cite{ambrosio2005gradient}. For example, the SVGD method can be viewed as the gradient flow of the relative entropy under the Stein metric. The geodesics, on the other hand, is the curve that takes the least cost under the corresponding metric to transform a probability measure to another one. 
The structure of geodesics under Wasserstein distance has been well studied \cite{villani2009optimal,santambrogio2015optimal}. For example, in the case of Wasserstein-2 distance, one may solve the Brenier potential through the Monge-Ampere equation so that the transport plan can be computed explicitly. As soon as one knows the transport plan, the geodesics has a clear description (see also Proposition \ref{prop:geodesic of ot} below), and the particles move with constant velocities and non-intersecting particle trajectories \cite{santambrogio2015optimal}. As a generalization of Wasserstein distance, WFR has also been proved to be a good metric in natural language processing to measure the distance between different documents, in waveform based earthquake location models to measure the discrepancy between true and synthetic earthquake signals, and in cellular population models to help to estimate cellular growth and death rates \cite{wang2020robust,zhou2018wasserstein,schiebinger2019optimal}, etc. The structure of geodesics under WFR is however less understood compared to the Wasserstein distance. For example, how the transportation and weight change balance and how the behaviors of the particles along the geodesics are unclear. We are interested in the geodesics under spherical WFR metric and seek for solutions in particle sense.

An important application of transforming the probability measures is sample generation, which is an important research direction in data science \cite{goodfellow2014generative,kingma2013auto, kingma2019introduction,rezende2015variational}. Compared with generative adversarial networks (GANs) \cite{goodfellow2014generative,arjovsky2017wasserstein} and variational autoencoders (VAE) \cite{kingma2013auto, kingma2019introduction}, the flow-based models directly focus on original data space and tries to evolve the density distribution by transportation fields \cite{tabak2010density, rezende2015variational}. 
In the continuous normalizing flow (CNF) \cite{chen2018neural}, the probability distribution is evolved according to the mass transportation under a velocity field. An advantage of CNF is that the inverse mapping is easy to obtain through solving related ODE backward. 
 However, the velocity field that transforms a given probability measure to a target one is not unique in general. Combining with optimal transport, Onken et. al. \cite{onken2020ot} proposed OT-Flow, an improved version of the CNF. The OT-Flow may be preferred in applications since the particles follow straight lines and the trajectories will not cross each other, and thus such a model is expected to improve the invertibility of the model and reduce the computational cost.

In this paper, we are interested in developing a deep learning framework to compute the geodesics under the spherical WFR metric. This framework can not only tell us how the particles and distribution evolve in the ``optimal way'' but also can be used as new generative models for \textit{weighted} samples. Following the framework for the OT-Flow \cite{onken2020ot},  we will use dynamic formulation of spherical WFR, where a source term is added to the transport equation. Inside this process, the particles not only are transported but also will carry an evolving weight. Using this framework, we are not only able to compute the geodesics $\rho_t$ but also compute the velocity field and the source term that realize this measure evolution, thus the geodesics.

If we set the start distribution as the standard normal distribution and the desired distribution as the target distribution, the learned velocity field and source for the geodesics can lead to particle motions and weight change for the desired distribution, thus a generative model for the target distribution. Our framework is clearly a generalization of the CNF and OT-Flow models in the sense that it considers weight change. This framework might be promising in dealing with weighted data, which can be costly for CNF models as one needs to resample. For example, the attention mechanism provides words with different weights \cite{vaswani2017attention,devlin2018bert}. Another particularly suited example is the Bayesian inference \cite{apte2007sampling, wu2022ensemble}, in which drawing samples from posterior is an essential task in estimation. Since directly calculating posterior with Bayesian formula to sample can be costly, we can rely on our new model to learn posterior from given weighted data and generate new samples for Bayesian inference. If one expects the transportation effect not to be significant in applications, such a framework may also be beneficial. For instance in pharmaceutical people use generative model to create and design new drug molecule and one may desire the new drug molecule to keep most already known structures.  

The rest of the paper is organized as follows. Section \ref{sec:pre} is devoted to a brief review of unbalanced optimal transport and WFR metrics that are useful to us later. We also introduce two flow-based generative models: CNF and OT-Flow, as a starting point of our framework. Then, we review some knowledge about geodesics and derive basic equations for geodesics under spherical WFR metric in section \ref{sec:geodesics}, which guides us to construct particle algorithms. In section \ref{sec:math_uot}, we develop a deep learning framework to compute the geodesics under spherical WFR metric. We illustrate how to use the learned geodesics to generate weighted samples efficiently in section \ref{sec: geodesics for sampling}, especially in the Bayesian framework. In section \ref{sec:numerics}, we provide some numerical experiments to validate the framework. We conclude the work and make a discussion in section \ref{sec:dis}.

\section{Preliminaries}\label{sec:pre}
In this section, we first collect some basics of unbalanced optimal transport including the dynamic and static formulations of Wasserstein-Fisher-Rao metric, as well as spherical WFR we will use in the rest of paper. We then give a brief introduction to continuous normalizing flows and OT-Flow model in section \ref{subsec: sample generation}, which are generated models based on transportation purely.

\subsection{The unblanced optimal transport and the WFR metrics}


Optimal Transport (OT) has a long history since Monge first posed the problem in 1781 \cite{monge1781memoire}, which sits at the intersection of various mathematical fields including probability, geometry, PDEs and optimization \cite{villani2009optimal,santambrogio2015optimal}. In recent years, optimal transport has seen an increasing amount of attention from computer science \cite{rubner2000earth,peyre2019computational,arjovsky2017wasserstein}, biological sciences \cite{schiebinger2019optimal,yang2020predicting}, economics \cite{galichon2018optimal,galichon2021survey}, etc. The optimal transport problem (or Kantorovich problem) is, given two distribution $\mu$ and $\nu$ and a cost function $c:X \times Y \rightarrow[0, \infty] $, one is supposed to solve
\begin{equation*}
\min \left\{\int_{X \times Y} c\, d \gamma \mid \gamma \in \Pi(\mu, \nu)\right\}, 
\end{equation*}
where $\Pi(\mu, \nu)$ is the set of {\it transport plans}, i.e. a joint measure on $X \times Y$, with marginal distribution $\mu$ and $\nu$.
The optimal value of Kantorovich problem with $c(x,y)=|x-y|^p$ is used to define Wasserstein-$p$ distance ($p \geq 1$) between $\mu$ and $\nu$ :
\begin{equation*}
W_{p}(\mu, \nu)=\left(\inf _{\gamma \in \Pi(\mu, \nu)} \int|x-y|^{p} d \gamma\right)^{1 / p}.
\end{equation*}
We can define the space $\mathcal{W}_{p}:=\{\mu \ \text{is a probability measure}|\int |x|^{p} \mu(dx)< \infty\}$. Then $(\mathcal{W}_{p},W_{p})$ is a complete metric space. Furthermore, $W_{p}$ distance admits the following dynamic formulation \cite[Chap. 5 Theorem 5.28]{santambrogio2015optimal}: On a convex and compact domain $\Omega$, $\mu$ and $\nu$ are two probability distribution on $\Omega$, $v_t$ is a vector field on $\Omega$,
\begin{equation}\label{eq:dynamic_ot}
W_{p}^{p}(\mu, \nu)=\min_{\rho, v}\left\{\int_{0}^{1}\left\|v_{t}\right\|_{L^{p}(\rho)}^{p} d t: \partial_{t} \rho_{t}+\nabla \cdot\left(\rho_{t} v_{t}\right)=0, \rho_{0}=\mu, \rho_{1}=\nu\right\},
\end{equation}
where $\left\|v_{t}\right\|_{L^{p}(\rho_{t})}^{p}=\int_{\Omega}\left|v_{t}(x)\right|^{p} \rho_{t}(dx)$. We will focus on $p=2$ in the rest of our article. 
The classical optimal transport theory reveals two useful facts. The first is that the optimal particle velocity under $W_2$ is a constant along one trajectory, which implies the trajectory is a straight line. The second property is that trajectories will not cross each other. These properties are beneficial for us to build an inverse mapping from $\nu$ to $\mu$. A brief discussion about geodesics under Wasserstein distance will be performed in section \ref{sec:geodesics}.


A notable restriction of optimal transport is that it is only defined between measures having the same mass, which might not be suitable for applications in image classification where measures need not be
normalized \cite{pele2008linear, rubner1997earth}, and biophysical phenomena involving some sort of mass
creation or destruction \cite{schiebinger2019optimal}. Recently, Wasserstein-Fisher-Rao (WFR) distance is proposed to generalize optimal transport to measures with different masses, which interpolates between the quadratic Wasserstein and the Fisher-Rao metrics \cite{ chizat2018interpolating,chizat2018unbalanced, kondratyev2016new}.
We first introduce the dynamic formulation of unbalanced optimal transport designed for the WFR distance, which leads to our proposed numerical method in section \ref{sec:math_uot}. By introducing a source term in the continuity equation, WFR distance relaxes the equality of mass constraint in the dynamic Benamou-Brenier formulation of optimal transport
\begin{multline}\label{eq:WFR}
\mathrm{d}_{\mathrm{WFR},\alpha}^2(\mu, \nu)=\inf _{\rho, v, g} \Bigg\{\int_{0}^{1} \int_{\Omega}\left(\frac{1}{2}|v_t(x)|^{2}+\frac{\alpha}{2} g_t^{2}(x)\right) \rho_{t}(dx)  d t:\\
\partial_{t} \rho_t+\nabla \cdot\left(\rho_t v_t\right)=\rho_t g_t, \rho_{0}=\mu, \rho_{1}=\nu\Bigg\},
\end{multline}
where $(\rho_t)_{t\in [0,1]} $ is a time-dependent density that interpolates between $\rho_0$ and $\rho_1$, $(v_t)_{t\in [0,1]} $ is a velocity field that describes the movement of mass and $(g_t)_{t\in[0,1]}$ is a scalar field (source term) associated with mass creation and destruction. $\alpha$ is a hyper-parameter to balance the effects of transport and mass creation/destruction explicitly.

WFR also admits a static Kantorovich formulation analogously to standard optimal transport \cite{chizat2018unbalanced}. Let $\mathcal{M}_{+}(X)$ be the space of nonnegative Radon measures on a compact set $X \subset \R^d$. For a measure $\pi\in\mathcal{M}_{+}(X\times X)$, its two marginals are denoted by $(\mathrm{Proj}_0)_{\#}\pi$ and $(\mathrm{Proj}_1)_{\#}\pi$ and are defined for any Borel set $A$ via
    \[
    (\mathrm{Proj}_0)_{\#}\pi(A)=\pi(A\times X), \quad (\mathrm{Proj}_1)_{\#}\pi(A)=\pi(X\times A).
    \]
\begin{definition}
\begin{itemize}
    \item (Semi-couplings)\, For $\mu,\nu\in\mathcal{M}_{+}(\Omega)$, the corresponding set of semi-couplings is 
\[
\Gamma(\mu,\nu):=\left\{(\pi_0,\pi_1)\in \left(\mathcal{M}_{+}(\Omega\times\Omega)\right)^2:(\mathrm{Proj}_0)_{\#}\pi_0=\mu, (\mathrm{Proj}_1)_{\#}\pi_1=\nu\right\}.
\]

\item (Cost function)\, A cost function is a function
\begin{gather*}
c:\begin{split}
&(\Omega\times[0,\infty))^2\to[0,\infty]\\
&(x_0,m_0),(x_1,m_1)\mapsto c(x_0,m_0,x_1,m_1)
\end{split}  
\end{gather*}
which is lower semi-continuous. in all its arguments and jointly positively 1-homogeneous and convex in mass variables $(m_0,m_1)$. This function $c(x_0, m_0, x_1, m_1)$ determines the cost of transporting a quantity of mass $m_0$ from $x_0$ to a (possibly different) quantity $m_1$ at $x_1$.
\item For a cost function $c$ we introduce the functional
\[
J_c(\pi_0,\pi_1):=\int_{\Omega\times\Omega}c\left(x,\frac{\pi_0}{\pi},y,\frac{\pi_1}{\pi}\right)d\pi(x,y),
\]
where $\pi\in\mathcal{M}_{+}(\Omega\times\Omega)$ is any measure such that $\pi_0,\pi_1\ll\pi$. This functional is well-defined since $c$ is jointly 1-homogeneous w.r.t. the mass variables.  
\end{itemize}
\end{definition}
\begin{proposition}[Static formulation of WFR by semi-couplings \cite{chizat2018unbalanced}]\label{pro:swfrstatic}
The WFR metric admits a static formulation characterized by semi-couplings:
\[
\mathrm{d}_{\mathrm{WFR},\alpha}^2(\mu, \nu)=\min_{(\pi_0,\pi_1)\in\Gamma(\mu,\nu)} J_c(\pi_0,\pi_1),
\]
where $c(x_0,m_0,x_1,m_1)=2\alpha \left(m_0+m_1-2\sqrt{m_0m_1}\cdot\overline{\cos}(|x_0-x_1|/2\sqrt{\alpha})\right)$ and the truncated cosine $\overline{\cos}(z)=\cos(|z|\wedge\frac{\pi}{2})$.
\end{proposition}

Recall the dynamic formulation \eqref{eq:WFR}. If $g$ is restricted to the family of zero mean, then one obtains spherical Wasserstein-Fisher-Rao (SWFR) distance (a.k.a, spherical Hellinger-Kantorovich distance) \cite{laschos2019geometric,kondratyev2019spherical}. Such $g$ can keep $\rho$ as a probability measure. For convenience, we use the short-hand notation $\bar{g}_t=\int_{\Omega} g_t\,d\rho_t$. SWFR distance used in the rest of paper has the following form:
\begin{multline}\label{eq:sWFR}
\mathrm{d}_{\mathrm{SWFR},\alpha}^2(\mu, \nu)=\inf _{\rho, v, g} \Bigg\{\int_{0}^{1} \int_{\Omega}\left(\frac{1}{2}|v_t|^{2}+\frac{\alpha}{2} \left(g_t-\bar{g}_t\right)^{2}\right) \rho_{t}(dx) d t:\\
\partial_{t} \rho_t+\nabla \cdot\left(\rho_t v_t\right)=\rho_t (g_t-\bar{g}_t), \rho_{0}=\mu, \rho_{1}=\nu\Bigg\}.
\end{multline}
It is remarked that the space $(\mathcal{M}_+(\Omega),\mathrm{d}_{\mathrm{WFR}})$ can be identified with the cone over the space of probability measures $(\mathcal{P}(\Omega),\mathrm{d}_{\mathrm{SWFR}})$, due to the scaling property of WFR metric \cite{laschos2019geometric, brenier2020optimal}. The geodesic between two probability measures in $(\mathcal{P}(\Omega),\mathrm{d}_{\mathrm{SWFR}})$ can be obtained by abstract projection from the cone $(\mathcal{M}_+(\Omega),\mathrm{d}_{\mathrm{WFR}})$ to the spherical space $(\mathcal{P}(\Omega),\mathrm{d}_{\mathrm{SWFR}})$, namely by renormalizing the mass and by rescaling of arclength parameter. (see more details in \cite{laschos2019geometric} and references therein.)  


Moreover, the WFR metrics (including SWFR metric) corresponds formally to a Riemannian structure on the space of measures, which generalizes the formulation of ordinary optimal transport. The transport, creation and destruction of mass between two measures can be described by this metric.


\subsection{Sample generation based on particle transportation}\label{subsec: sample generation}

Let us briefly review the continuous normalizing flow (CNF) model \cite{chen2018neural} and OT-Flow model \cite{onken2020ot}, which are sample generative models based on particle transportation purely.

CNF aims to build a continuous and invertible mapping between an arbitrary distribution $\rho_{0}$ and standard normal distribution $\rho_{1}$. Alternatively, for a given time $T$, we are trying to obtain a mapping $z: \R^{d} \times [0,T]\rightarrow \R^{d}$. The mapping $z$ defines a continuous evolution $x \mapsto z(x, t) $ of every $x \in \R^{d}$, which can be viewed as trajectory of particles. Then the density $\rho(z(x,t),t)$ satisfies
\begin{equation}\label{original density formulation}
    \log \rho_{0}(x)=\log \rho(z(x,t),t)+\log|\det \nabla z(x,t)| \quad \text{for all} \quad x \in \R^{d}.
\end{equation}
Especially at time $T$ we have $\log \rho_{0}(x)=\log \rho_{1}(z(x,T),T)+\log|\det \nabla z(x,T)| $. Define $\ell(x, t):=\log |\det \nabla z(x, t)|$, then $z(x,t)$ and $\ell(x,t)$ satisfy the following ODE system
\begin{equation}\label{eq:original ODE}
\partial_{t}\left[\begin{array}{c}
z(x, t) \\
\ell(x, t)
\end{array}\right]=\left[\begin{array}{c}
v(z(x, t), t ; \boldsymbol{\theta}) \\
\operatorname{tr}(\nabla v(z(x, t), t ; \boldsymbol{\theta}))
\end{array}\right], \quad\left[\begin{array}{c}
z(x, 0) \\
\ell(x, 0)
\end{array}\right]=\left[\begin{array}{c}
x \\
0
\end{array}\right].
\end{equation}
where the second ODE can be derived from the first one. For convenience we solve them together to obtain the change of $\rho$, which will lead to a more efficient estimation of density. Following is the derivation of second ODE in \eqref{eq:original ODE}:
\begin{equation*}
\begin{split}
\frac{\partial \ell(x, t)}{\partial t}  &= \frac{1}{\det (\nabla z(x,t))} \frac{\partial \det (\nabla z(x,t))}{\partial t} \\
&= \frac{1}{\det (\nabla z(x,t))} \cdot \det (\nabla z(x,t)) \cdot \mathrm{tr} \left[ (\nabla z(x,t))^{-1} \frac{\partial \nabla z(x,t)}{\partial t}  \right]\\
&=\frac{1}{\det (\nabla z(x,t))} \cdot \det (\nabla z(x,t)) \cdot \mathrm{tr} \left[ (\nabla z(x,t))^{-1} \nabla z(x,t)  \nabla v(z(x,t),t)    \right] \\
&=\mathrm{tr} \left[ ( \nabla v(z(x,t),t) \right],
\end{split}
\end{equation*}
where we have used following identities
\begin{equation*}
     \frac{\partial \det (A)}{\partial t}=\det (A) \cdot \mathrm{tr} \left[  A^{-1}  \frac{\partial A}{\partial t}\right], \quad \mathrm{tr} (AB)= \mathrm{tr} (BA).
\end{equation*}

From the ODE system \ref{eq:original ODE} we can see that if we have a velocity field, then we can track the evolution and obtain the final distribution at time $T$. Thus we can set the velocity field as an output of neural network and minimize the KL divergence between target distribution and final distribution obtained from ODE. Once the velocity field is learned, one may run the ODE backward so that the transport map can be inverted. The invertibility of CNF provides us with access to estimating density of sample space, which can be used for density estimation and Bayesian inference.

 In general, the velocity field that transforms a given probability measure to a target one is not unique in the formulation of CNF. To tackle this problem, Onken et al. \cite{onken2020ot} proposed an improved version of CNF: OT-Flow, which leverages optimal transport theory to regularize the CNF and enforce straight trajectories that are easier for numerical integration. More precisely, OT-Flow designs the following cost functional to train the velocity field
\begin{equation}\label{eq:Cost1}
    J=\mathrm{D}_{\mathrm{KL}}\left[\rho(x, T) \| \rho_{1}(x)\right]+\gamma_{1}\mathbb{E}_{\rho_{o}} \left[\int_{0}^{T} \frac{1}{2}|v(z(x, t), t)|^{2} d t \right].
\end{equation}

The first part in \eqref{eq:Cost1} is a KL divergence term as a soft terminal constraint, which enforces the terminal distribution $\rho(x,T)$ transported by velocity field to get close to $\rho_1$. The second term is related to $W_{2}$ distance in optimal transport theory, which can also be regarded as a penalty of the squared arc-length of the trajectories. Ideally if the KL divergence term is zero, minimizing the cost function is equivalent to minimizing $W_{2}$ distance and solving the optimal velocity field, which will lead to two useful properties mentioned above, encouraging straight trajectory. Here $\gamma_{1}$ is a hyper-parameter to balance KL divergence and trajectory penalty.

Above cost function cannot be used to train directly since we only have discrete samples as $\rho_{0}$. We would like to use Monte-Carlo to approximate the cost function, which requires us to rewrite cost function in the form of expectation over $\rho_{0}$. After simplifying the KL divergence term with density relationship \eqref{original density formulation} and dropping constant in the formulation, we get final cost function $J$ as following
\begin{equation}\label{eq:OT-Flow}
\begin{split}
    J&=\mathbb{E}_{\rho_{o}(x)} \left[ C(x,T)+\gamma_{1}L(x,T) \right],\\
    C(x,T)&=-\ell(x,T) +\frac{1}{2}|\boldsymbol{z}(\boldsymbol{x}, T)|^{2}+\frac{d}{2} \log (2 \pi),\\
    L(x,T)&=\int_{0}^{T} \frac{1}{2}|v(z(x, t), t)|^{2} d t.
\end{split}
\end{equation}
OT-Flow can be regarded as a model to learn geodesics under Wasserstein distance. We will adopt a similar formulation to develop our deep learning framework for geodesics under the spherical WFR metric.

\section{Geodesics under spherical WFR}\label{sec:geodesics}
In this section we will first review knowledge about geodesics based on \cite[Chap. 5]{santambrogio2015optimal}. Then we derive basic equations for geodesics under spherical WFR. The formulations will be used to construct particle algorithm and new regularization based on the velocity field.

\subsection{Geodesics}\label{subsec:geodesics}
In a metric space $(X,d)$, we define the length of a curve $\omega: [0,1] \rightarrow X$ as
\[
    \text{Length}(\omega):=\sup \left\{\sum_{k=0}^{n-1} d\left(\omega\left(t_{k}\right), \omega\left(t_{k+1}\right)\right): n \ge 1,0=t_{0}<t_{1}<\cdots<t_{n}=1 \right\}.
\]
which is also the total variation of $\omega$. For all absolutely continuous functions $\omega$ in X , we have Length$(\omega)<+\infty$ and Length$(\omega)=\int_{0}^{1}|\omega^{\prime}|(t)dt$.

\begin{definition}
    \begin{itemize}
        \item (Geodesics) A curve $\omega: [0,1] \rightarrow X$ is said to be a geodesic between $x_{0}$ and $x_{1} \in X$ if it minimizes the length among all curves such that $\omega(0)=x_{0}$ and  $\omega(1)=x_{1}$.
        \item (Length space) A space $(X,d)$ is said to be a length space if it holds:
    \[
        d(x, y)=\inf \{\operatorname{Length}(\omega): \omega \in \mathrm{AC}(X), \omega(0)=x, \omega(1)=y\}.
    \]
        \item (Constant-speed geodesic) In a length space, a curve $\omega: [0,1] \rightarrow X$ is said to be a constant-speed geodesic between $\omega(0)$ and $\omega(1)$ if it satisfies:
\begin{equation}\label{constant geodesics}
    d(\omega(t), \omega(s))=|t-s| d(\omega(0), \omega(1)) \quad \text { for all } t, s \in[0,1].
\end{equation}
    \end{itemize}
\end{definition}





The following result describes the structure of geodesics under Wasserstein distance. 
\begin{proposition}[Theorem 5.27 in \cite{santambrogio2015optimal}]\label{prop:geodesic of ot}
Suppose $\Omega$ is convex.$(\mathcal{W}_{p},W_{p})$ is the metric space. $\mu, \nu \in \mathcal{W}_{p}$ and $\gamma \in \Pi(\mu, \nu)$ is an optimal transport plan for the cost $c(x, y)=|x-y|^{p}(p \geq 1) .$ Let $\pi_{t}(x, y)=(1-$ $t) x+$ ty. Then the curve $\mu_{t}=\left(\pi_{t}\right)_{\#} \gamma$ is a constant-speed geodesics in $\mathcal{W}_{p}$ connecting $\mu_{0}=\mu$ and $\mu_{1}=\nu$. As a consequence, the space $\mathcal{W}_p(\Omega)$ is a length space.
\end{proposition}

According to the description of the geodesics, a particle that is transported from $x$ to $y$ moves under the evolution $\pi_t(x,y)=(1-t)x+ty$, which indicates that the particle is moving with constant velocity $y-x$. Moreover, in the optimal transport plan, the particles will not cross each other so the trajectories are straight line that are non-intersecting.

 Though WFR is a generalization of Wasserstein distance, the structure of the geodesics under spherical WFR metric and corresponding particle motions are not as clear as the Wasserstein case. In particular, while the geodesics for Wasserstein distance can be obtained easily using the optimal plan in the static formulation by Proposition \ref{prop:geodesic of ot}, the one for spherical WFR cannot be obtained directly if one knows the optimal solution to the static formulation in Propostion \ref{pro:swfrstatic}.
 
 As a start, we note the following simple observation for the dynamic formulation.
\begin{proposition}\label{pro:wfrgeo}
    The optimal $(\rho_t)_{t\in [0,1]}$ satisfying dynamic description of spherical WFR \eqref{eq:sWFR} is exactly the constant speed geodesics under spherical WFR.
\end{proposition}

\begin{proof}
Consider the optimal solution $\rho_{t}, g_{t}, v_{t}, t \in [0,1]$ satisfying minimization problem. For $\beta_{1},\beta_{2} \in [0,1], \beta_{1}<\beta_{2}$, we have
\begin{gather*}
    \begin{split}
        &\int_{\beta_{1}}^{\beta_{2}} \int_{\Omega}\left(\frac{1}{2}|v_t|^{2}+\frac{\alpha}{2} \left(g_t-\bar{g}_t\right)^{2}\right) \rho_t(dx) d t
        = \min _{\rho, v, g} \Bigg\{\int_{\beta_{1}}^{\beta_{2}} \int_{\Omega}\left(\frac{1}{2}|v|^{2}+\frac{\alpha}{2} (g-\bar{g})^{2}\right) \rho(dx) d t:\\
        &\hspace*{15em}  
        \partial_{t} \rho+\nabla \cdot\left(\rho v\right)=\rho (g-\bar{g}), \rho(\beta_{1})=\rho_{\beta_{1}}, \rho(\beta_{2})=\rho_{\beta_{2}}\Bigg\}\\
        = &\frac{1}{\beta_{2}-\beta_{1}}\min _{\rho_{\tau}, v_{\tau}, g_{\tau}} \Bigg\{\int_{0}^{1} \int_{\Omega}\left(\frac{1}{2}|v_{\tau}|^{2}+\frac{\alpha}{2} \left(g_{\tau}-\bar{g_{\tau}}\right)^{2}\right) \rho_{\tau}(dx) d \tau:\\ & \hspace*{12em}  
        \partial_{\tau} \rho_{\tau}+\nabla \cdot\left(\rho_{\tau} v_{\tau}\right)=\rho_{\tau} (g_{\tau}-\bar{g_{\tau}}),
       \rho_{\tau}(0)=\rho_{\beta_{1}}, \rho_{\tau}(1)=\rho_{\beta_{2}}\Bigg\}\\
        = &\frac{1}{\beta_{2}-\beta_{1}}\mathrm{d}_{\mathrm{SWFR}}^{2}(\rho_{\beta_{1}},\rho_{\beta_2}).
    \end{split}
\end{gather*}
The first equality follows directly from minimization problem. In fact, if it has a better solution $(\tilde{\rho}_t, \tilde{v}_t, \tilde{g}_t)_{\beta_{1}\le t\le \beta_{2}}$ such that the objective function is smaller, then we may concatenate with $(\rho_t, v_t, g_t)_{t\in[0,\beta_{1}]\cup [\beta_{2},1]}$ to get a better solution for the original minimization problem. Note that the concatenated solution is still continuous due to the boundary condition, so the continuity equation still holds in weak sense. The second equality is a simple time rescaling $\tau=\frac{1}{\beta_{2}-\beta_{1}} (t-\beta_1)$. Correspondingly, one has the following correspondence: $\rho_{\tau}(\tau,x)=\rho(t,x), v_{\tau}(x,\tau)=(\beta_{2}-\beta_{1}) v(x,t), g_{\tau}(x,\tau)=(\beta_{2}-\beta_{1}) g(x,t)$ for $\tau \in [0,1]$. Thus we have
\begin{equation}\label{eq:distancerelation}
    \begin{aligned}
        \frac{1}{\beta_2}\mathrm{d}_{\mathrm{SWFR}}^{2}(\rho_{0},\rho_{\beta_2})=\frac{1}{\beta_1}\mathrm{d}_{\mathrm{SWFR}}^{2}(\rho_{0},\rho_{\beta_1})+\frac{1}{\beta_2-\beta_1}\mathrm{d}_{\mathrm{SWFR}}^{2}(\rho_{\beta_1},\rho_{\beta_2}).
    \end{aligned}
\end{equation}
Note that $\mathrm{d}_{\mathrm{SWFR}}(\rho_{0},\rho_{\beta_{2}}) \leq \mathrm{d}_{\mathrm{SWFR}}(\rho_{0},\rho_{\beta_{1}})+ \mathrm{d}_{\mathrm{SWFR}}(\rho_{\beta_{1}},\rho_{\beta_{2}})$, thus we have
\begin{equation*}
    \begin{aligned}
    \frac{\beta_{2}-\beta_{1}}{\beta_{1}\beta_{2}}\mathrm{d}_{\mathrm{SWFR}}^{2}(\rho_{0},\rho_{\beta_{1}})+\frac{\beta_{1}}{(\beta_{2}-\beta_{1})\beta_{2}}&\mathrm{d}_{\mathrm{SWFR}}^{2}(\rho_{\beta_{1}},\rho_{\beta_{2}})-\\
    &\frac{2}{\beta_{2}}\mathrm{d}_{\mathrm{SWFR}}(\rho_{0},\rho_{\beta_{1}})\cdot\mathrm{d}_{\mathrm{SWFR}}(\rho_{\beta_{1}},\rho_{\beta_{2}}) \leq 0.
    \end{aligned}
\end{equation*}
In fact the left side $\geq 0$. Thus the equal sign holds only when
\begin{equation}\label{eq:disrelation2}
    \begin{aligned}
        \frac{\mathrm{d}_{\mathrm{SWFR}}^{2}(\rho_{0},\rho_{\beta_{1}})}{\beta_{1}^{2}}=\frac{\mathrm{d}_{\mathrm{SWFR}}^{2}(\rho_{\beta_{1}},\rho_{\beta_{2}})}{(\beta_{2}-\beta_{1})^{2}}.
    \end{aligned}
\end{equation}
Setting $\beta_2=1$ in \eqref{eq:distancerelation} and using \eqref{eq:disrelation2}, one finds that $\mathrm{d}_{\mathrm{SWFR}}(\rho_{0},\rho_{\beta_{1}})=\beta_1\mathrm{d}_{\mathrm{SWFR}}(\rho_{0},\rho_{1})$. It follows then by \eqref{eq:disrelation2} that 
$\mathrm{d}_{\mathrm{SWFR}}(\rho_{\beta_2},\rho_{\beta_1})
=(\beta_2-\beta_1)\mathrm{d}_{\mathrm{SWFR}}(\rho_0, \rho_1)$ for $\forall \beta_{1},\beta_{2} \in [0,1], \beta_{1}<\beta_{2}$.
By definition \eqref{constant geodesics} we know that the optimal $(\rho_t)_{t\in [0,1]}$ is the constant speed geodesic under spherical WFR.
\end{proof}

\subsection{Equations for the density and potential for the spherical WFR metric}

Below, we make use of Proposition \ref{pro:wfrgeo} to investigate the geodesics under spherical WFR metric.
Consider the following optimization problem
\begin{gather}\label{eq:mini}
\min_{\rho,v,g} \left\{\frac{1}{2} \int_0^1\int_{\mathbb{R}^d} \rho |v|^2+\alpha \rho g^2 \,dx\,dt, \partial_t\rho + \nabla \cdot (\rho v)=\rho g, \int_{\mathbb{R}^d} \rho g \,dx=0\right\},
\end{gather}
where $\alpha$, the source coefficient, is to balance the effects of transport and creation/destruction of mass explicitly. The Lagrangian function for \eqref{eq:mini} is
\[
\mathcal{L} =  \frac{1}{2}\int_0^1\int_{\mathbb{R}^d} \rho |v|^2+\alpha \rho g^2 \,dx\,dt-\int_0^1\int_{\mathbb{R}^d} \Phi(x,t) (\partial_t\rho + \nabla \cdot (\rho v)-\rho g)\,dx\,dt-\int_0^1 \gamma(t)\left(\int_{\mathbb{R}^d} \rho g \,dx\right)dt,
\]
where $\Phi(x,t)$ and $\gamma(t)$ are all Lagrange multipliers. Taking the variation, by the first order optimality conditions, one has
\begin{equation}\label{eq:optimal_condition}
\left\{
\begin{aligned}
&\frac{\delta \mathcal{L}}{\delta v}=0, \\
& \frac{\delta \mathcal{L}}{\delta g}=0, \\
& \frac{\delta \mathcal{L}}{\delta \rho}=0.
\end{aligned}
\right.
\Longrightarrow
\left\{
\begin{aligned}
&v=-\nabla \Phi, \\
& g=-\frac{1}{\alpha}(\Phi-\gamma(t)), \\
& \partial_t\Phi=\frac{1}{2}|\nabla \Phi|^2+\frac{1}{2\alpha}(\Phi-\gamma(t))^2.
\end{aligned}
\right.
\end{equation}
Using the fact $\int \rho g\, dx=0$, one can determine $\gamma(t)=\int \Phi \,d\rho=:\bar{\Phi}(t)$. Then the evolution of $\rho$ is governed by
\begin{equation}\label{eq:weight_PDE}
\partial_t\rho -\nabla \cdot (\rho \nabla \Phi)=-\frac{1}{\alpha}\rho (\Phi-\bar{\Phi}).
\end{equation}
To summarize, we get the equations of $\rho$ and $\Phi$
\begin{equation}\label{eq:uotequations}
\left\{
\begin{aligned}
& \partial_t\rho -\nabla \cdot (\rho \nabla \Phi)=-\frac{1}{\alpha}\rho (\Phi-\bar{\Phi}),\\
& \partial_t\Phi=\frac{1}{2}|\nabla \Phi|^2+\frac{1}{2\alpha}(\Phi-\bar{\Phi})^2.
\end{aligned}
\right.
\end{equation}
We remark that if we define a Hamiltonian 
\[
H(\rho,\Phi)=\frac{1}{2}\int_0^1\int_{\R^d}\rho|\nabla\Phi|^2+\frac{1}{\alpha}\rho(\Phi-\bar{\Phi})^2 \,dx\,dt,
\]
then system  \eqref{eq:uotequations}  can be rewritten as
\begin{equation*}
\left\{
\begin{aligned}
& \dot \rho = -\frac{\delta H}{\delta \Phi},\\
& \dot \Phi = \frac{\delta H}{\delta \rho}.
\end{aligned}
\right.
\end{equation*}

The Hamiltonian structure for the minimizer is in fact a consequence of the well-known Pontryagin Maximum Principle \cite{evans1983introduction}. The equation that governs the evolution of $\Phi$ is also called Hamilton-Jacobi equation, i.e. the second equation in \eqref{eq:uotequations}.


\subsection{The evolution of distribution and particles along the geodesics}

In order to learn the geodesics under spherical WFR, we will derive related equations of weight evolution, density formula and velocity field evolution first. We use $z: \R^{d} \times [0,T]\rightarrow \R^{d}$ to represent the trajectory of particles. Consider a weighted formulation
\begin{equation}\label{integral weight}
    \rho(z,t)=\int_{\mathbb{R}^d} w(x,t) \delta(z-z(x,t))p(x)d x
\end{equation}
for some measure $p$ so that $\rho_{0}(x)=p(x)w(x,0)$. We have the following claims about density and particle velocity along the trajectories.
\begin{theorem}
\begin{enumerate}[(i)]
    \item The particle weight satisfies
    \begin{gather}\label{eq:weight}
        w(x, t)=w(x,0) e^{-\frac{1}{\alpha}\int_0^t (\Phi(z(x,s),s)-\bar{\Phi}(s))ds},
    \end{gather}
    which directly leads to density formula under spherical WFR as
    \begin{gather}\label{eq:density_uot}
        \rho_0(x)e^{-\frac{1}{\alpha}\int _{0}^{t}(\Phi(z(x,s),s)-\bar{\Phi}(s))ds}=\rho(z(x,t),t)\cdot \mathrm{det} (\nabla z(x,t)).
    \end{gather}
    \item  For the particle velocity under spherical WFR, it holds that
    \begin{gather}\label{eq:veolocity and weight}
        v(z(x,t),t)=v(x,0)e^{\frac{1}{\alpha}\int_0^t(\Phi(z(x,s),s)-\bar{\Phi}(s))ds}=\frac{v(x,0)w(x,0)}{w(x,t)}.
    \end{gather}
Consequently, the particles move in straight lines.
    \end{enumerate}
\end{theorem}

\begin{proof}
\begin{enumerate}[(i)]
    \item  By taking \eqref{integral weight} into \eqref{eq:weight_PDE}, we get 

\begin{equation*}\label{weight and position system}
\left\{
\begin{aligned}
& \frac{d}{dt}z(x,t)=-\nabla \Phi(z(x,t),t),\\
& \frac{d}{dt}w(x,t)=-\frac{1}{\alpha}\left(\Phi(z(x,t),t)-\bar{\Phi}(t)\right)w(x,t).
\end{aligned}
\right.
\end{equation*}
Then one can solve $w(x,t)$ as 
\begin{gather*}
        w(x, t)=w(x,0) e^{-\frac{1}{\alpha}\int_0^t (\Phi(z(x,s),s)-\bar{\Phi}(s))ds}.
\end{gather*}
Consider the conservation of mass, for $\forall \, \Omega \in \mathcal{R}^{d}$, we have:
\begin{equation*}
    \int_{\Omega} \rho_{0}(x) d x =\int_{z(\Omega,t)} \rho(z(x,t),t) d z.
\end{equation*}
By the change of variables formula and \eqref{integral weight}, it follows that 

    \[
    \rho_0(x)e^{-\frac{1}{\alpha}\int _{0}^{t}(\Phi(z(x,s),s)-\bar{\Phi}(s))ds}=\rho(z(x,t),t)\, \text{det} (\nabla z(x,t)).
    \]

    \item Recall \eqref{eq:optimal_condition} from Lagrangian equation of spherical WFR, taking the gradient of the Hamilton-Jacobi equation, one has 
\[
    \partial_t v+v\cdot \nabla v-\frac{1}{\alpha}(\Phi-\bar{\Phi}) v=0.
\]
We use the classical argument of characteristic lines. Let $\gamma(s;x,t)$ be the characteristic line which satisfies
\begin{equation*}
    \left\{
\begin{aligned}
&\frac{d\gamma(s;x,t)}{ds}=v(\gamma(s;x,t),s), \\
&\gamma(t;x,t)=x,
\end{aligned}
\right.
\end{equation*}
then $U(s):=v(\gamma(s;x,t),s)$ satisfies

\begin{equation*}
    U^{'}(s)=\partial_tv+v\cdot \nabla v=\frac{1}{\alpha}(\Phi-\bar{\Phi})v(\gamma(s;x,t),s)=\frac{1}{\alpha}(\Phi-\bar{\Phi})U(s).
\end{equation*}
It follows that
\begin{equation*}
    v(x,t)=U(t)=U(0)e^{\frac{1}{\alpha}\int_0^t(\Phi(\gamma(s;x,t),s)-\bar{\Phi}(s))ds}.
\end{equation*}
Note that $\gamma(s;z(x,t),t)=z(x,s)$ and $\gamma(0;z(x,t),t)=x$. Hence 
\[
        v(z(x,t),t)=v(x,0)e^{\frac{1}{\alpha}\int_0^t(\Phi(z(x,s),s)-\bar{\Phi}(s))ds}=\frac{v(x,0)w(x,0)}{w(x,t)}.
\]
Since the direction of the velocity does not change, the particles move in straight lines, which concludes the proof.
\end{enumerate}
\end{proof}

\begin{remark}
$\rho_0(x)e^{-\frac{1}{\alpha}\int _{0}^{t}(\Phi(z(x,s),s)-\bar{\Phi}(s))ds}$ is a probability density with respect to $x$ for every $t\ge 0$.
\end{remark}

The formula \eqref{eq:veolocity and weight} indicates that, in the optimal case, the magnitude of particle velocity is inversely proportional to its weight while the direction of velocity remains the same along each trajectory.

\section{Learning geodesics under spherical WFR}\label{sec:math_uot}

In this section, we propose a deep learning framework to learn the geodesics under spherical WFR metric. A KL divergence term based on the inverse map is used for the terminal condition. Moreover, a new regularization term based on particle velocity and weight is introduced in section \ref{subsec:regularization} as a substitute for the Hamilton-Jacobi equation for the potential in dynamic formulation. Then, in section \ref{subsec:total cost}, we obtain the total cost for our model and make some discussion on the hyper-parameter $\gamma_1$. Detailed implementation of the algorithm is showed in section \ref{sec:detail_algo}.


\subsection{Imposing the terminal condition}\label{subsec:difficulty}

Suppose we are given a starting distribution $\rho_0$, and we may know or may not know the expression. In the latter case, we may assume some samples drawn from $\rho_0$. We are also given the terminal distribution $\rho_1$ up to potentially an unknown normalizing constant.

To compute the geodesics, we may impose the terminal condition $\rho(\cdot, t=1)=\rho_1$. We propose in this work to construct the KL divergence using the inverse map. In particular, we consider the probability distribution evolved from the terminal distribution $\rho_1$, denoted by $\tilde{\rho}_0$. Then, compare it to the initial distribution $\rho_0$. 

Let $\widetilde{\rho}(z,T-t):=\rho(z,t)$. One can deduce from \eqref{eq:uotequations} that
\begin{gather}\label{eq:inverse flow}
\left\{
\begin{aligned}
&\partial_t\widetilde{\rho}(z,t)+\nabla \cdot (\widetilde{\rho}(z,t)\nabla\Phi(z,T-t))=\frac{1}{\alpha}\widetilde{\rho}(z,t)(\Phi(z,T-t)-\hat{\Phi}(T-t)),\\
&\widetilde{\rho}(z,0)=\rho_1(z), \quad \hat{\Phi}(T-t)=\int \Phi(z,T-t)\widetilde{\rho}(z,t)dz.
\end{aligned}
\right.
\end{gather}

Similarly, using $x(z,t)=z(x,T-t)$, it holds that
\begin{gather}
    \rho_1(z)e^{\frac{1}{\alpha}\int _{0}^{T}(\Phi(x(z,t),T-t)-\hat{\Phi}(T-t))dt}=\widetilde{\rho}(x(z,T),T)\, \mathrm{det} (\nabla_z x(z,T)),
\end{gather}
which is
\begin{gather}
    \widetilde{\rho}(x,T)e^{-\frac{1}{\alpha}\int _{0}^{T}(\Phi(z(x,t),t)-\hat{\Phi}(t))dt}=\rho_1(z(x,T)) \,\mathrm{det} (\nabla_x z(x,T)).
\end{gather}
We turn to minimize the KL divergence between $\widetilde{\rho}(x,T)$ and $\rho_0(x)$ 
\begin{gather}\label{new KL equation}
\begin{split}
    &\mathrm{D}_{\mathrm{KL}}[\rho_0(x)\|\widetilde{\rho}(x,T)]=\int_{\mathbb{R}^d}\text{log}\left(\frac{\rho_0(x)}{\widetilde{\rho}(x,T)}\right)\rho_0(x)dx\\
    &=\int_{\R^d}\text{log}\left(\frac{\rho_0(x)e^{-\frac{1}{\alpha}\int _{0}^{T}(\Phi(z(x,t),t)-\hat{\Phi}(t))dt}}{\rho_1(z(x,T))\cdot \text{det}(\nabla_x(z(x,T)))}\right)\rho_0(x)dx\\
    &= \int_{\mathbb{R}^d} [-\text{log}(\rho_1(z(x,T))-\log \det(\nabla z(x,T))]\rho_0(x)dx\\
    &\quad \quad -\frac{1}{\alpha}\int_{\mathbb{R}^d} \Big(\int_0^T(\Phi(z(x,t),t)-\hat{\Phi}(t))dt\Big)\rho_0(x)dx+ \int_{\mathbb{R}^d} \text{log}(\rho_0(x))\rho_0(x)dx.
\end{split}
\end{gather}
The term $\int\log(\rho_0(x))\rho_0(x)dx$ can be dropped, since it is unrelated to the parameters to optimize.

We now explain why we use a new KL divergence based on the inversed map to impose terminal condition, instead of directly using KL divergence between $\rho(\cdot,t=1)$ and $\rho_{1}$. Mimicking the KL divergence approach in the CNF framework and using \eqref{eq:density_uot}, we have
\begin{gather}
  \begin{split}
  & \mathrm{D}_{\mathrm{KL}}[\rho(z(x,T))\| \rho_1(z)] =\int_{\mathbb{R}^d} \log\left(\frac{\rho(z(x,T))}{\rho_1(z(x,T))}\right)\rho(z(x,T))dz\\
             &= \int_{\mathbb{R}^d} [-\text{log}(\rho_1(z(x,T))-\log(\text{det}(\nabla z(x,T))]\rho_0(x)e^{-\frac{1}{\alpha}\int _{0}^{T}(\Phi(z(x,t),t)-\bar{\Phi}(t))dt} dx \\
             &\quad \quad -\frac{1}{\alpha}\int_{\mathbb{R}^d} \left(\int_0^T(\Phi(z(x,t),t)-\bar{\Phi}(t))dt\right)\rho_0(x)e^{-\frac{1}{\alpha}\int _{0}^{T}(\Phi(z(x,t),t)-\bar{\Phi}(t))dt} dx\\
             &\quad \quad +
             \int_{\mathbb{R}^d} \text{log}(\rho_0(x))\rho_0(x)e^{-\frac{1}{\alpha}\int _{0}^{T}(\Phi(z(x,t),t)-\bar{\Phi}(t))dt} dx\\
             &=:I_1+I_2+I_3.
  \end{split} \label{kl}
\end{gather}
Clearly, $I_1$ and $I_2$ can be approximated using the standard Monte Carlo approximation. However, $I_3$ cannot be computed easily if the expression of $\rho_0$ is unknown or complicated, and we cannot drop it since it is not a constant with respect to model parameters.

There are also two alternative options to address the problem of estimating $I_{3}$. We tried them in solving above problems brought by weight change, while we decided not to follow these approaches due to the drawbacks discussed below.
\begin{itemize}
\item {\bf Estimating  initial density}
 One possible strategy is to use the variational expression of $\log(\rho_0(x)/\rho(x))$ introduced in  \cite{johnson2019framework} to approximate $\log (\rho_0)$. In fact,
\[
\log(\rho_0/\rho)=\text{argmin}_{D^{\prime}}[\mathbb{E}_{x\sim \rho_0}\text{log}(1+e^{-D^{\prime}(x)})+\mathbb{E}_{x\sim \rho}\text{log}(1+e^{D^{\prime}(x)})],
\]
where the argument of the minimization problem $D'$ is a function of $x$.
In practice, one may take some class of functions for the optimization and use the optimal one $\mathcal{D}(x)$ from this class to approximate the theoretical optimum. Taking $\rho(x)$ as the standard normal distribution,  $\log(\rho_0)$ can be then computed as
\begin{equation}\label{eq:log_rho0}
    \text{log}(\rho_0(x))=-\frac{1}{2}|x|^2-\frac{d}{2}\text{log}(2\pi)+\mathcal{D}(x).
\end{equation}

\item {\bf Kernelized Stein Discrepancy (KSD)}
The second potential approach is replacing KL divergence with another weak metric such as the kernelized Stein discrepancy \cite{liu2016kernelized, chwialkowski2016kernel,gorham2017measuring} or the Wasserstein distances, to avoid the density estimation. 
\end{itemize}

The approach to approximate $\log(\rho_0)$ by \eqref{eq:log_rho0} works in some toy cases but is not appealing as a building block in our algorithm. First, in high dimension case such estimation can be tricky and costly. Second, estimating data density is one of essential applications of CNF models. Though approximating $\rho_0$ out first is feasible, we would prefer solutions without estimating $\rho_{0}$ to build our framework. For the alternative KSD metric, we were not able to obtain satisfactory result when the initial density's support is not connected. \cite{hu2018stein} argues that when KSD is small, it means that within the region of generated samples, the score function of $s_p$ matches the target score function $s_q$ well. An almost-zero empirical KSD does not necessarily imply capturing all the modes or recovering all the support of the true density. Considering these reasons and results of numerical experiments we did, we choose to construct KL divergence of $\rho_{0}$ using inverse map in our framework eventually.

\subsection{Regularization}\label{subsec:regularization}

For OT-Flow, Onken et al. \cite{onken2020ot} used HJB equation of $\Phi$ to construct a regularization to help training and accelerate convergence. We hope to construct a regularization similarly to obtain a better velocity field in training. Instead of using the Hamilton-Jacobi equation, we make use of the relationship between velocity field and weight in spherical WFR \eqref{eq:veolocity and weight} to impose
\begin{gather}\label{regularization Rv}
\begin{split}
    \mathbf{R}_{v} &:=\int_0^T\int_{\mathbb{R}^d}\left|v(z(x,t),t)w(x,t)-v(x,0)w(x,0)\right|^2\rho_0(x)e^{-\frac{1}{\alpha}\int_0^t(\Phi(z(x,s),s)-\bar{\Phi}(s))ds}\,dx\,dt \\
    &=\int_0^T\int_{\mathbb{R}^d}\left|v(z(x,t),t)w(x,t)-v(x,0)w(x,0)\right|^2\rho_0(x)\frac{w(x,t)}{w(x,0)}\,dx\,dt
\end{split}
\end{gather}
as regulariztion. Such a term penalizes the velocity field along the trajectory, which can lead to a better velocity field suited for our framework in training.

\subsection{Total cost}\label{subsec:total cost}

We are supposed to learn the geodesics under spherical WFR. In our framework, once the potential $\Phi$ is known, the velocity field, the source field and thus geodesic curve can be computed. Thus we parameterize $\Phi$ as an output of a neural network, and total cost function with regularization can be written as:
\begin{equation}\label{eq:total_loss}
    J(\Phi)=\mathrm{D}_{\mathrm{KL}}[\rho_0(x)\| \widetilde{\rho}(x,T)]+\gamma_1\cdot\int_{0}^{T} \int_{\mathbb{R}^{d}} \left(\frac{1}{2} |\nabla\Phi|^{2}  +\frac{1}{2\alpha}  (\Phi-\bar{\Phi})^{2} \right) \rho(d z) d t+\gamma_2\cdot\mathbf{R}_{v},
\end{equation}
where $\gamma_1$ and $\gamma_2$ are the hyper-parameters to balance the terms in cost function $J$. Since we will sample from data space later to approximate the value by Monte-Carlo, we rewrite the integral into the form of expectation over initial density for convenience. By \ref{new KL equation} the KL term can be rewritten as
\begin{multline*}
    \mathrm{D}_{\mathrm{KL}}[\rho_0(x)\| \widetilde{\rho}(x,T)]=\mathbb{E}_{\rho_{0}(x)} \Bigg[-\log(\rho_1(z(x,T))-\log \det(\nabla z(x,T)) \\-\frac{1}{\alpha} \int_0^T(\Phi(z(x,t),t)-\hat{\Phi}(t))dt\Bigg]
    +\mathbb{E}_{\rho_{0}(x)}\Big[ \log(\rho_{0}(x))\Big].
\end{multline*}
where $\hat{\Phi}(t)$ is defined in \eqref{eq:inverse flow} and the detail of the implementation will be explained in subsection \ref{sec:detail_algo}.
The second part can be rewritten as:
\begin{multline*}
\gamma_1\cdot\int_{0}^{T} \int_{\mathbb{R}^{d}} \left(\frac{1}{2} |\nabla\Phi|^{2}  +\frac{1}{2\alpha}  (\Phi-\bar{\Phi})^{2} \right) \rho(d z) d t \\
= \frac{\gamma_{1}}{2}\,\mathbb{E}_{\rho_{0}(x)} \Bigg[ \int_0^T \Big( |\nabla\Phi(z(x,t),t)|^{2}+\frac{1}{\alpha} (\Phi(z(x,t),t)-\bar{\Phi}(t))^2 \Big) \cdot e^{-\frac{1}{\alpha}\int _{0}^{t}(\Phi(z(x,s),s)-\bar{\Phi}(s))d s}\, d t \Bigg].
\end{multline*}

The third term is obvious from \eqref{regularization Rv}:
\[
    \gamma_{2}\cdot\mathbf{R}_{v}=\gamma_{2}\,\mathbb{E}_{\rho_{0}(x)} \Bigg[\int_0^T|\nabla\Phi(z(x,t),t)w(x,t)-\nabla\Phi(x,0)w(x,0)|^2\frac{w(x,t)}{w(x,0)}dt \Bigg].
\]

\subsection*{Discussion on the hyper-parameter $\gamma_1$}

The bigness of the parameter $\gamma_{1}$ is important. 
As we hope the constraint $\rho(\cdot, t=1)=\rho_1$ to be a hard constraint (i.e.  the KL divergence should be zero), we choose $\gamma_{1}$ to be relatively small. In fact, if we ignore the regularization term and consider parameter $\gamma_{1}$ alone. It may be shown theoretically that if $\gamma_{1}\rightarrow 0$, the optimal solution will converge to the geodesics. 
In particular, the following proposition tells us why this is a reasonable approach.
\begin{proposition}\label{prop:small_gamma}
Consider the optimization problem
\[
\min_{u\in X} E(u)+\gamma F(u).
\]
If both functionals are convex, lower semi-continuous and $U:=\mathrm{argmin}(E)$ is nonempty, then as $\gamma\to 0^+$, a cluster point of the minimizers $\{ u(\gamma)\}$ is a solution to the following problem
\[
\min_{u\in U} F(u).
\]
\end{proposition}
For the proof of this proposition, one may set $E_*=\inf E(u)>-\infty$ since $U$ is assumed to be nonempty. 
Then, consider $J_{\alpha}=F(u)+\alpha (E(u)-E_*)$
and $J_{\infty}=F(u)+\mathbf{1}(U)$, where the indicator function 
\[
\mathbf{1}(U)=\begin{cases}
0 & x\in U\\
+\infty & x\notin U.
\end{cases}
\]
Then, one may verify the $\Gamma$-convergence of $J_{\alpha}$ to $J_{\infty}$ as $\alpha\to\infty$ \cite{braides2002gamma, chizat2018unbalanced}. 
This proposition suggests that if we choose $\gamma_1$ small enough, our framework can give a good approximation to the geodesics under the spherical WFR metric.

As an additional remark, if the parameter $\gamma_1$ is big, a single step is like an implicit Euler step for the gradient descent of the KL divergence, which is the generalization of Jordan-Kinderlehrer-Otto (JKO) scheme \cite{de1993new,jordan1998variational,li2022computational}. 

\subsection{Detailed algorithm and implementation}\label{sec:detail_algo}

In practical computations, we can only use discrete particles to approximate the high dimensional distribution. Consider empirical measure for particle system: 
\[
    \rho(x,t)=\sum_{i=1}^{n}w_i(t)\delta(x-x_i(t)),
\]
where $w_i(t)$ denotes the weight of particle $x_i$ at time $t$. 
Taking the empirical measure into \eqref{eq:weight_PDE} and considering the second ODE in \eqref{eq:original ODE}, one obtains following ODE system
\begin{equation}\label{discrete ODE system}
\partial_{t}\left[\begin{array}{c}
z_{i}(t) \\
w_{i}(t)\\
\ell(x_{i}, t)
\end{array}\right]=\left[\begin{array}{c}
-\nabla \Phi(z_i(t),t) \\
-\frac{1}{\alpha}(\Phi(z_i(t),t)-\bar{\Phi}(t))w_i(t)\\
-\operatorname{tr}(\nabla^{2} \Phi(z_{i}(t), t )) \\
\end{array}\right], \quad\left[\begin{array}{c}
z_{i}(0) \\
w_{i}(0)\\
\ell(x_{i},0)
\end{array}\right]=\left[\begin{array}{c}
x_{i} \\
w_{0}(x_{i})\\
0
\end{array}\right].
\end{equation}
where $\bar{\Phi}$ is approximated by
\begin{gather}\label{eq:barphiapprox}
\bar{\Phi}(t)=\frac{1}{n}\sum_{i=1}^n w_i(t)\Phi(z_i(t), t)
\end{gather}
and we identify $z_i(t)$ with $z(x_i,t)$ as the $i^{\mathrm{th}}$ particle's trajectory and $w_{0}(x_{i})$ as $i^{\mathrm{th}}$ particle's initial weight. Assume that $\sum_{i=1}^n w_i(0)=n$, then $\sum_{i=1}^n w_i(t)\equiv n$ for any $t>0$.
Alternatively, one may choose any convenient $\bar{\Phi}(t)$ to integrate \eqref{discrete ODE system} and then normalize $\{w_i(t)\}$ for the sum to be $n$ so that the true $\bar{\Phi}(t)$ can be approximated using \eqref{eq:barphiapprox}.

The first term in \eqref{eq:total_loss} is the KL divergence
\begin{multline*}
    \mathrm{D}_{\mathrm{KL}}[\rho_0(x)\| \widetilde{\rho}(x,T)]=\mathbb{E}_{\rho_{0}(x)} \Bigg[-\log(\rho_1(z(x,T))-\log \det(\nabla z(x,T)) \\-\frac{1}{\alpha} \int_0^T(\Phi(z(x,t),t)-\hat{\Phi}(t))dt\Bigg]
    +\mathbb{E}_{\rho_{0}(x)}\Big[ \log(\rho_{0}(x))\Big].
\end{multline*}
We can drop the term $\mathbb{E}_{\rho_{0}(x)}\Big[ \log(\rho_{0}(x))\Big]$ which does not affect the training, thus
\begin{gather*}
\begin{split}
    J_{\mathrm{KL}}&:=\mathbb{E}_{\rho_{o}(x)} \Bigg[- \log \rho_1(z(x,T)) -\ell(x,T) +\frac{1}{\alpha}\int_0^T(\Phi(z(x,t),t)-\hat\Phi(t))dt\Bigg]\\
    & \approx\frac{1}{n} \sum_{i=1}^{n}\Bigg[ -\ell(x_i,T)-\log \rho_1(z(x_{i},T))+\frac{1}{\alpha}\int_0^T \Phi(z(x_{i},t),t) dt\Bigg]-\frac{1}{\alpha}
    \int_0^T \hat\Phi(t)\,dt.
\end{split}
\end{gather*}
To compute $\int_0^T \hat{\Phi}(t)\,dt=\int_0^T \hat{\Phi}(T-t)\,dt$, assuming $\{\hat{Z}_{j}\}_{j=1}^{N}$ are samples from $\rho_{1}$, we solve
\begin{equation}\label{inverse discrete ODE system}
\partial_{t}\left[\begin{array}{c}
\hat{x}_{j}(t) \\
\hat{w}_{j}(t)\\
\end{array}\right]=\left[\begin{array}{c}
\nabla \Phi(\hat{x}_j(t),T-t) \\
\frac{1}{\alpha}(\Phi(\hat{x}_j(t),T-t)-\hat{\Phi}(T-t))\hat{w}_j(t)\\
\end{array}\right], \quad\left[\begin{array}{c}
\hat{x}_{j}(0) \\
\hat{w}_{j}(0)\\
\end{array}\right]=\left[\begin{array}{c}
\hat{Z}_{j} \\
1\\
\end{array}\right].
\end{equation}
Here $\hat{\Phi}$ can be calculated as
\begin{gather}\label{eq:hatPhiapprox}
\hat\Phi(T-t) = \frac{1}{N}\sum_{j=1}^{N} \hat{w}_{j}(t)\Phi(\hat{x}_{j}(t),T-t).
\end{gather}
Alternatively,  one may solve the ODE system \eqref{discrete ODE system} backward in time with terminal conditions to approximate $\hat\Phi$. Again in \eqref{inverse discrete ODE system}, the function $\hat{\Phi}(T-t)$ can be replaced by any other convenient function (e.g. 0) and one needs only to normalize $\hat{w}_j(t)$ and then use \eqref{eq:hatPhiapprox} for the approximation.

The second term in \eqref{eq:total_loss} characterizes spherical WFR distance
\begin{equation*}
\begin{aligned}
    J_{\mathrm{SWFR}}&:=\int_{0}^{T} \int_{\mathbb{R}^{d}} \left(\frac{1}{2} \rho |v|^{2}  +\frac{1}{2} \alpha \rho  g^{2} \right) d z\, d t\\
    &=\frac{1}{2}\,\mathbb{E}_{\rho_{0}(x)} \Bigg[ \int_0^T \Big( |\nabla\Phi(z(x,t),t)|^{2}+\frac{1}{\alpha} (\Phi(z(x,t),t)-\bar{\Phi}(t))^2 \Big) e^{-\frac{1}{\alpha}\int _{0}^{t}(\Phi(z(x,s),s)-\bar{\Phi}(s))d s}\, d t \Bigg]\\
    &\approx \frac{1}{2n} \sum_{i=1}^{n}\Bigg[ \int_{0}^{T} \left( |\nabla\Phi(z(x_i,t),t)|^{2}+\frac{1}{\alpha} (\Phi(z(x_{i},t),t)-\bar\Phi(t))^{2} \right) \frac{w_i(t)}{w_i(0)} \,d t\Bigg],
\end{aligned}
\end{equation*}
where $\bar\Phi(t)$ is approximated by \eqref{eq:barphiapprox}.

The regularization term in \eqref{eq:total_loss} is:
\begin{equation*}
\begin{aligned}
    J_{\mathrm{R}}:=\mathbf{R}_{v}
    &=\mathbb{E}_{\rho_{0}(x)} \Bigg[\int_0^T|\nabla\Phi(z(x,t),t)w(x,t)-\nabla\Phi(x,0)w(x,0)|^2\frac{w(x,t)}{w(x,0)}\,dt \Bigg]\\
    & \approx \frac{1}{n}\sum_{i=1}^{n}\Bigg[\int_{0}^{T}  |\nabla\Phi(z(x_i,t),t)w_{i}(t)-\nabla\Phi(x_i,0)w_i(0)|^2 \frac{w_i(t)}{w_i(0)}\,d t\Bigg].
\end{aligned}
\end{equation*}

We parameterize $\Phi$ with a neural network and use ADAM optimizer \cite{kingma2014adam} to minimize the cost function in \eqref{eq:total_loss}, which is $J= J_{\mathrm{KL}}+\gamma_1 J_{\mathrm{SWFR}}+\gamma_2J_{\mathrm{R}}$ indeed.  Then we formulate our deep learning framework for computing geodesics under spherical WFR in Algorithm \ref{alg:1}.

\begin{algorithm}
	\renewcommand{\algorithmicrequire}{\textbf{Notation:}}
	\renewcommand{\algorithmicensure}{\textbf{Require:}}
	\caption{Framework for computing geodesics under spherical WFR}
 	\begin{algorithmic}[1]
		\ENSURE  particles $\{x_{i}\}_{i=1}^n$ drawn from $\rho_0$, source parameter $\alpha $, $\log \rho_{1}$ up to a constant, time interval $\left[0,T \right]$, initializing network $\Phi$, hyper-parameters $\gamma_1, \gamma_2$
		\FOR{number of training iterations}
		\STATE Use $\Phi$ to solve the ODE system \eqref{discrete ODE system} to obtain position $z(x_i,t)$ and weight $w_i(t)$
        
        \STATE  Draw samples from $\rho_{1}$ and use $\Phi$ to solve the ODE system \eqref{inverse discrete ODE system} to calculate $\hat{\Phi}(t)$.
        
		\STATE Calculate cost function $J= J_{\mathrm{KL}}+\gamma_1 J_{\mathrm{SWFR}}+\gamma_2J_{\mathrm{R}}$ using $z(x_i,t)$, $w_i(t)$ and $\hat{\Phi}(t)$.
		\STATE Use ADAM optimizer optimizer to update network parameter of $\Phi$
		\ENDFOR
	\end{algorithmic}  
	\label{alg:1}
\end{algorithm}

\section{Using the geodesics to generate weighted samples}\label{sec: geodesics for sampling}

As we know, OT-Flow is developed based on previous CNF model, and can be viewed as a model using geodesics under Wasserstein distance between standard normal distribution and data distribution for sample generation. Similarly, by setting $\rho_{1}$ in our framework as a given distribution which is easy to sample (standard normal distribution for instance), we can use the learned geodesics to build a new generative model for weighted samples. In particular, this new generative model can be applied to given weighted data samples, gaining an advantage over previous CNF models. We will call our model UOT-gen (short for ``unbalanced OT-generation'') and present it in Algorithm \ref{alg:sample gen}.

\begin{algorithm}
	\renewcommand{\algorithmicrequire}{\textbf{Notation:}}
	\renewcommand{\algorithmicensure}{\textbf{Require:}}
	\caption{Framework of UOT-gen}
 	\begin{algorithmic}[1]
		\ENSURE  particles $\{x_{i}\}_{i=1}^n$ with corresponding initial weights $\{w_{i}\}_{i=1}^n$ drawn from $\rho_0$, source parameter $\alpha $, $\log \rho_{1}$ up to a constant, time interval $\left[0,T \right]$, initializing network $\Phi$, hyper-parameters $\gamma_1, \gamma_2$
		\FOR{number of training iterations}
		
        \STATE Use $\Phi$ to solve the ODE system \eqref{discrete ODE system} to obtain position $z(x_i,t)$ and weight $w_i(t)$.
        
        \STATE  Draw samples from $\rho_{1}$ and use $\Phi$ to solve the ODE system \eqref{inverse discrete ODE system} to calculate $\hat{\Phi}(t)$.
        
		\STATE Calculate cost function $J= J_{\mathrm{KL}}+\gamma_1 J_{\mathrm{SWFR}}+\gamma_2J_{\mathrm{R}}$ using $z(x_i,t)$, $w_i(t)$ and $\hat{\Phi}(t)$.
		\STATE Use ADAM optimizer to update network parameter of $\Phi$.
		\ENDFOR
		\STATE Draw equally weighted samples $\{\hat{Z}_{j}\}_{j=1}^N$ from $\rho_{1}$.
		\STATE Use well-trained $\Phi$ to solve the ODE system \eqref{inverse discrete ODE system} and get position $\hat{x}_{j}(T)$ and weight $\hat{w}_j(T)$. $\{\hat{x}_{j}(T)\}_{j=1}^N$ with weight $\{\hat{w}_{j}(T)\}_{j=1}^N$ are the generated weighted samples satisfying $\rho_{0}$.
	\end{algorithmic}  
	\label{alg:sample gen}
\end{algorithm}

The above generative model is particularly suited in the Bayesian framework, where the objective is to infer and sample the posterior distribution of unknown parameter $\theta\in\R^d$ given some observed data $\mathcal{D}$. By Bayes' theorem, the target distribution (i.e. the posterior distribution) is given by
\begin{equation}\label{eq:bayes}
    p(\theta|\mathcal{D})=\frac{p(\mathcal{D}|\theta)p(\theta)}{p(\mathcal{D})},
\end{equation}
through which one can draw samples and do density estimation for $\theta$. $p(\mathcal{D}|\theta)$ is the law of data given the parameters, or so-called likelihood of $\theta$, $p(\theta)$ is the prior distribution, and $Z:=p(\mathcal{D})$ is the normalization constant or the so-called evidence. In general, the normalization constant,
\[
Z=\int p(\mathcal{D}|\theta)p(\theta)d\theta
\]
is an intractable integral and one can only evaluate the numerator of \eqref{eq:bayes}. Often times calculating $p(\mathcal{D}|\theta)$ is costly. For instance in reinforcement learning, estimating $p(\mathcal{D}|\theta)$ requires running strategy, which can involve massive calculation.

We can sample from prior $p(\theta)$ with normalized $p(\mathcal{D}|\theta)$ as weight to generate new weighted samples satisfying posterior, i.e. we take $p$ in \eqref{integral weight} as prior and $w(\theta,0)$ as normalized $p(\mathcal{D}|\theta)$:
\begin{equation}
    \rho_{0}(\theta)=w(\theta,0)p(\theta).
\end{equation}
We set $\rho_{0}$ as weighted data mentioned above and $\rho_{1}$ as standard normal distribution. By training our UOT-gen model we can draw new weighted samples from posterior.

In most cases, the observation is continued and sequential. We are supposed to update posterior estimation and draw new posterior samples after obtaining new data. We would like to use previous estimation to help updating. An online algorithm can be developed based on UOT-gen model. If new data are observed, we can use generated samples with weight from last well-trained UOT-gen as prior and calculate corresponding likelihood. Then we multiply weight in generated samples with likelihood and normalized them as new weight. The last generated samples with new weight should satisfy updated posterior, which can be used for training to update parameters of UOT-gen model.

Suppose we have already trained an UOT-gen model for the latest posterior based on given observations. If new sequential data are observed, then the online algorithm can be as following in Algorithm \ref{alg:online}:
\begin{algorithm}
	\renewcommand{\algorithmicrequire}{\textbf{Notation:}}
	\renewcommand{\algorithmicensure}{\textbf{Require:}}
	\caption{Online algorithm for sequential data}
	\label{alg:online}
 	\begin{algorithmic}[1]
		\ENSURE  a well-trained UOT-gen model based on observed data so far
		\IF{new observations during a time period  $m \Delta t$ are made}
		\STATE drawing weighted samples from well-trained model: position $\theta_{i}$ and weight $w_i$
		\STATE Calculate corresponding likelihood $p(\mathcal{D}|\theta_{i})$ during new observation time period
		\STATE Normalize $p(\mathcal{D}|\theta_{i})w_{i}$ as new weight $\bar{w_{i}}$
		\STATE Using $\theta_{i}$ and $\bar{w_{i}}$ to train UOT-gen model with Algorithm \ref{alg:1}. The initial values of parameters in UOT-gen model can be set as the latest well-trained ones before observations are made.
		\ENDIF
	\end{algorithmic}  
\end{algorithm}

One advantage of online algorithm is that we do not need to calculate new likelihood by accumulating the whole time process. We just need to calculate the likelihood in the new time period of observation and multiply it with generated sample's weight. Meanwhile since the latest posterior is not far away from the previous one, we can set the initial values of parameters in UOT-gen model as last well-trained ones. Thus the new UOT-gen model can be trained with ease.

\section{Numerical Experiments}\label{sec:numerics}
In this section, we consider some test examples to demonstrate our deep learning framework and  UOT-gen model. We first test our method on a 1-D Gaussian mixture toy example to illustrate features of the geodesics under spherical WFR metric. The result also shows the effectiveness of the method in the sense that the particle system converges to the expected distribution precisely. We then perform density estimation on several two dimensional toy problems as done in \cite{onken2020ot}. We also apply the UOT-gen model to Bayesian problem to generate weighted samples satisfying posterior. Meanwhile, we test the online Algorithm \ref{alg:online} on the same Bayesian problem with sequential data, in which one is required to update our UOT-gen iteratively to adjust posterior estimation after obtaining new observations. Our UOT-gen and online algorithm achieve competitive prediction accuracy in numerical experiments.

 \begin{figure}[htbp]
  \centering
  \vspace{-1.5cm}
  \includegraphics[width=0.7\textwidth]{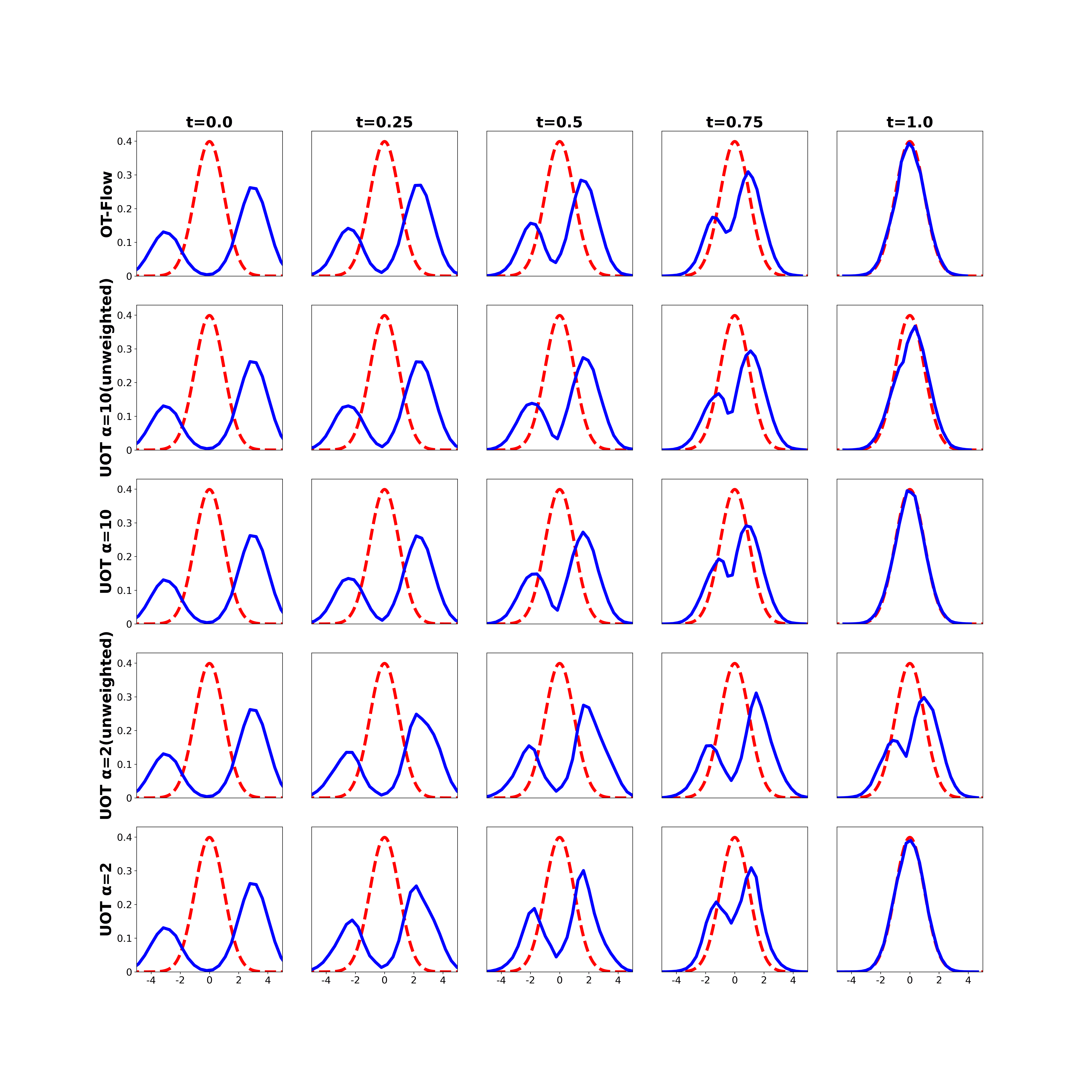}
  \vspace{-1cm}
    \caption
    {Evolution of particle distribution in forward transform with different $\alpha$. The first row reproduces results of OT-Flow model. The $2^{\mathrm{nd}}$ and $3^{\mathrm{rd}}$ rows show the particle distributions without weight (only transport is considered) and with weight respectively when $\alpha=10$. The $4^{\mathrm{th}}$ and $5^{\mathrm{th}}$ rows show the results with $\alpha=2$ and other settings are the same as the $2^{\mathrm{nd}}$ and $3^{\mathrm{rd}}$ rows.}
    \label{fig:forward_position}
\end{figure}
\begin{figure}[htbp]
  \centering
  \vspace{-2.5cm}
  \includegraphics[width=0.7\textwidth]{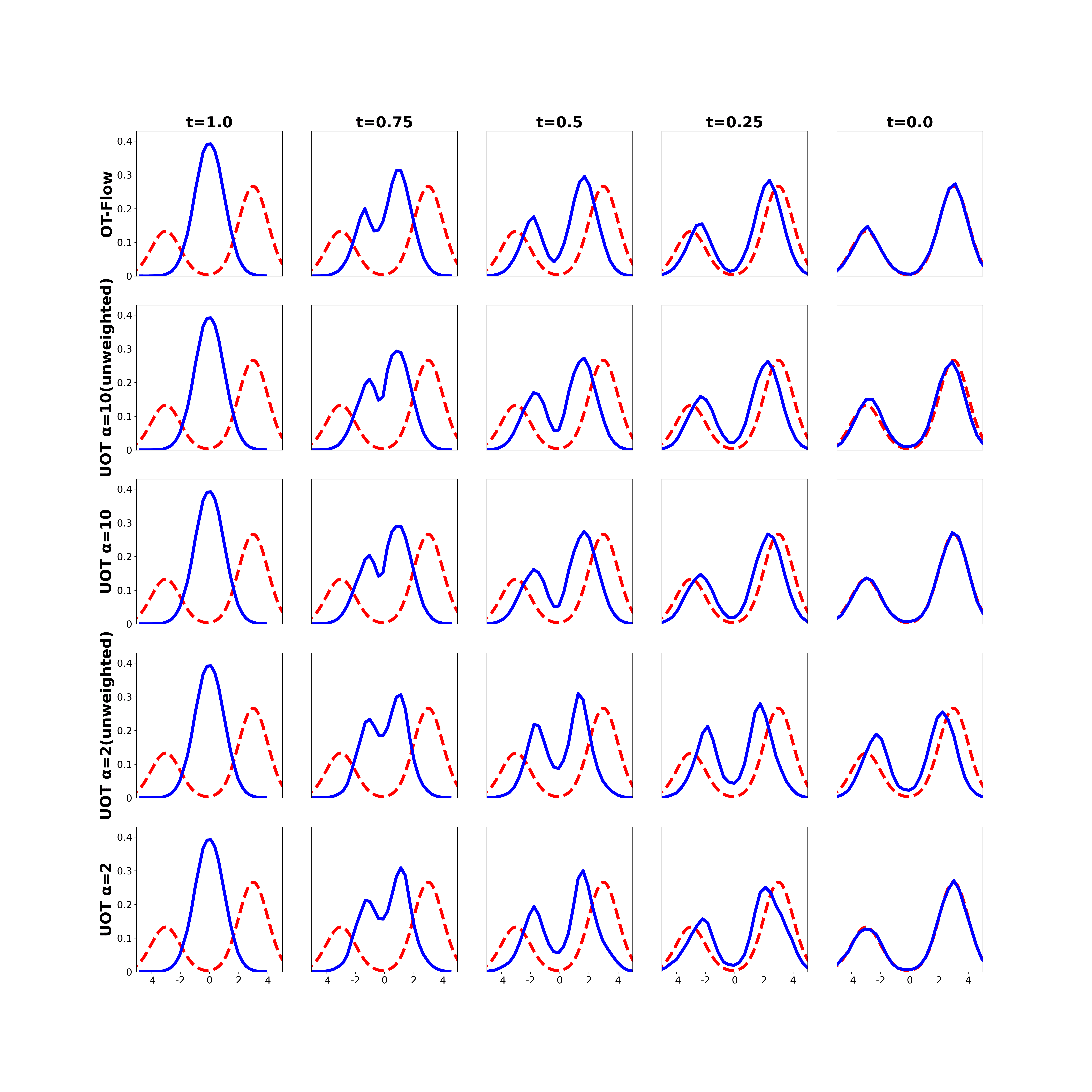} 
  \vspace{-1cm}
\caption{Evolution of particle distribution in inverse transform with different $\alpha$. Parameters and other settings are the same as in Figure \ref{fig:forward_position}.
}
    \label{fig:backward_position}
\end{figure}

\subsection*{Network architecture}
We adopt the same network architecture as in \cite{onken2020ot}. The potential function $\Phi$ is parameterized as
\[
\Phi(s;\theta)=\omega^\top N(s;\theta_N)+\frac{1}{2}s^\top(A^\top A)s+b^\top s+c,
\]
where $\theta=(\omega,\theta_N,A,b,c)$, $s=(x,t)\in\R^{d+1}$ are the imput features corresponding to space-time. $N(s;\theta_N):\R^{d+1}\to\R^m$ is a Residual Neural Network (ResNet) \cite{he2016deep}. $\theta$ consists of all the trainable weights: $\omega \in\R^m, \theta_N\in \R^p, A\in \R^{r\times (d+1)}, b\in\R^{d+1},c\in \R$. In our experiments, we set $r=d$. $A,b$ and $c$ model quadratic potentials (linear dynamics) whereas $N$ models the nonlinear dynamics \cite{ruthotto2020machine}.

\subsection*{ResNet}
For the nonlinear $N(s;\theta_N)$, we adopt a simple two layer ResNet in our experiments:
\begin{gather*}
  u_0=\sigma(K_0s+b_0),\\
  N(s;\theta_N)=u_1=u_0+\sigma(K_1u_0+b_1).
\end{gather*}

Here, $K_0\in\R^{m\times(d+1)}, K_1\in\R^{m\times m}, b_0,b_1\in\R^m$. We select
the element-wise activation function $\sigma(x)=\log(\exp(x)+\exp(-x))$ as in \cite{onken2020ot}, which is the antiderivative of the hyperbolic tangent, i.e., $\sigma'(x)=\text{tanh}(x)$. Therefore, hyperbolic tangent is the activation function of the flow $\nabla\Phi$. 

Also, we can compute the $\nabla_s\Phi(s;\theta)$ and $\text{tr}(\nabla^2\Phi(s;\theta))$ explicitly using the methods in \cite{onken2020ot}. For the forward propagation of flow, we use Runge-Kutta 4 with equidistant time steps to solve \eqref{eq:original ODE} and time integrals in cost function $J$.

\subsection*{Hyper-parameters}
In all experiments, we consider the time interval as $[0,1]$, i.e. $T=1$. The total time step is set as 8. The width of neural network $m$ is 32 in the experiments. We take $\gamma_1=\gamma_2=0.01$ as a default choice to make a balance for terms in the cost function.

\subsection{1-D Gaussian mixture}
As a first example, we use the 1-D Gaussian mixture problem for our deep learning framework.

\begin{figure}[htbp]
    \centering
    \vspace{-0.3cm}
    \includegraphics[width=0.8\textwidth]{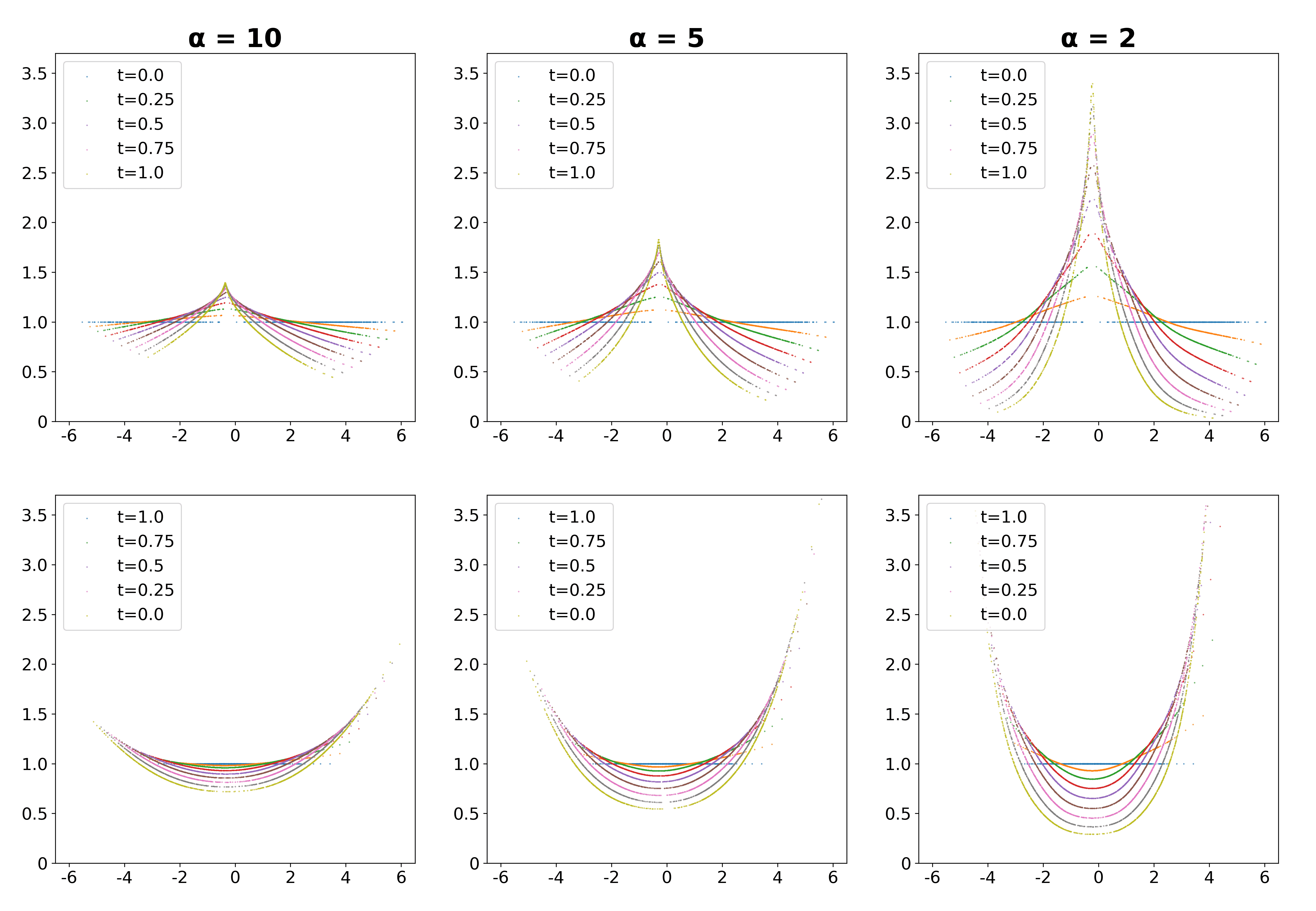}
    \vspace{-0.3cm}
\caption{Evolution of particle weights with different $\alpha$. In all figures, x-axis indicates the current position of particles and y-axis represents the weight of particles. The first and second row compare the weight change over time in forward and inverse flow respectively.}
\label{fig:weight}
\end{figure}

We sample from the Gaussian mixture
\begin{gather}\label{rho_0_1d}
\rho_0(x)= \frac{1}{3}\cdot\frac{1}{\sqrt{2\pi}}e^{-(x+3)^2/2}+\frac{2}{3}\cdot\frac{1}{\sqrt{2\pi}}e^{-(x-3)^2/2}
\end{gather}
 and use our framework to learn the geodesics from $\rho_0$ to standard normal distribution $\rho_1$.

Figure \ref{fig:forward_position} and Figure \ref{fig:backward_position} illustrate the evolution of particle distribution in the geodesics for forward and inverse transformations respectively, against different parameter $\alpha$. In all figures, red dash curves indicate target density function whereas blue curves are distributions of particles at each time $t$. By Figure \ref{fig:forward_position}, our deep learning framework shows compelling results with $\alpha=10,2$ (the $3^{\mathrm{rd}}$ and $5^{\mathrm{th}}$ rows) compared to OT-Flow (the $1^{\mathrm{st}}$ row) in the sense that our framework captures the target distribution $\rho_1$ precisely. The figures in $2^{\mathrm{nd}}$ and $4^{\mathrm{th}}$ rows show the particle distribution evolution without weight change. The differences between the $2^{\mathrm{nd}}$ and $3^{\mathrm{rd}}$ rows (as well as the $4^{\mathrm{th}}$ and $5^{\mathrm{th}}$ rows) show the effect brought by weight change. Compared with OT-Flow, the transportation of particles in our UOT-gen is weaker. The particles seem to be ``lazier'' for that they move to the location close to target distribution and the weight change compensate the remained. Compared to $\alpha=10$, the result with $\alpha=2$ shows a stronger birth-death effect (weight change) whereas weaker transport effect, which is consistent with the guiding PDE \eqref{eq:weight_PDE}. Figure \ref{fig:backward_position} shows the results of inverse transformation (starting with $\rho_1$) with other settings being the same as in Figure \ref{fig:forward_position}, which validates the efficiency of sample generation of our model.

\begin{figure}[htbp]
\centering
\vspace{-0.5cm}
\subfigure[Forward trajectories]{
\begin{minipage}[t]{0.45\linewidth}
\centering
\includegraphics[width=6cm]{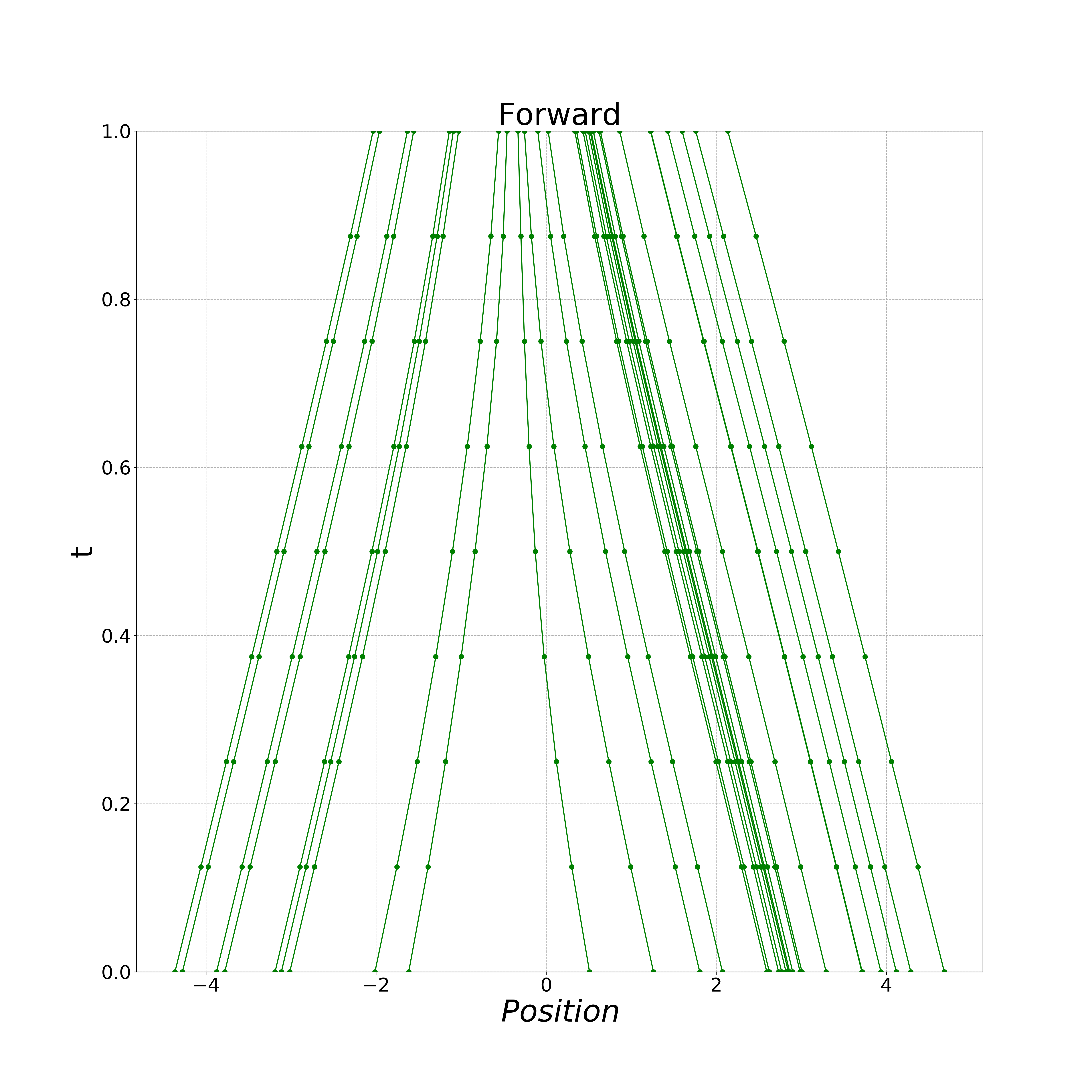}
\end{minipage}
}
\subfigure[Inverse trajectories]{
\begin{minipage}[t]{0.45\linewidth}
\centering
\includegraphics[width=6cm]{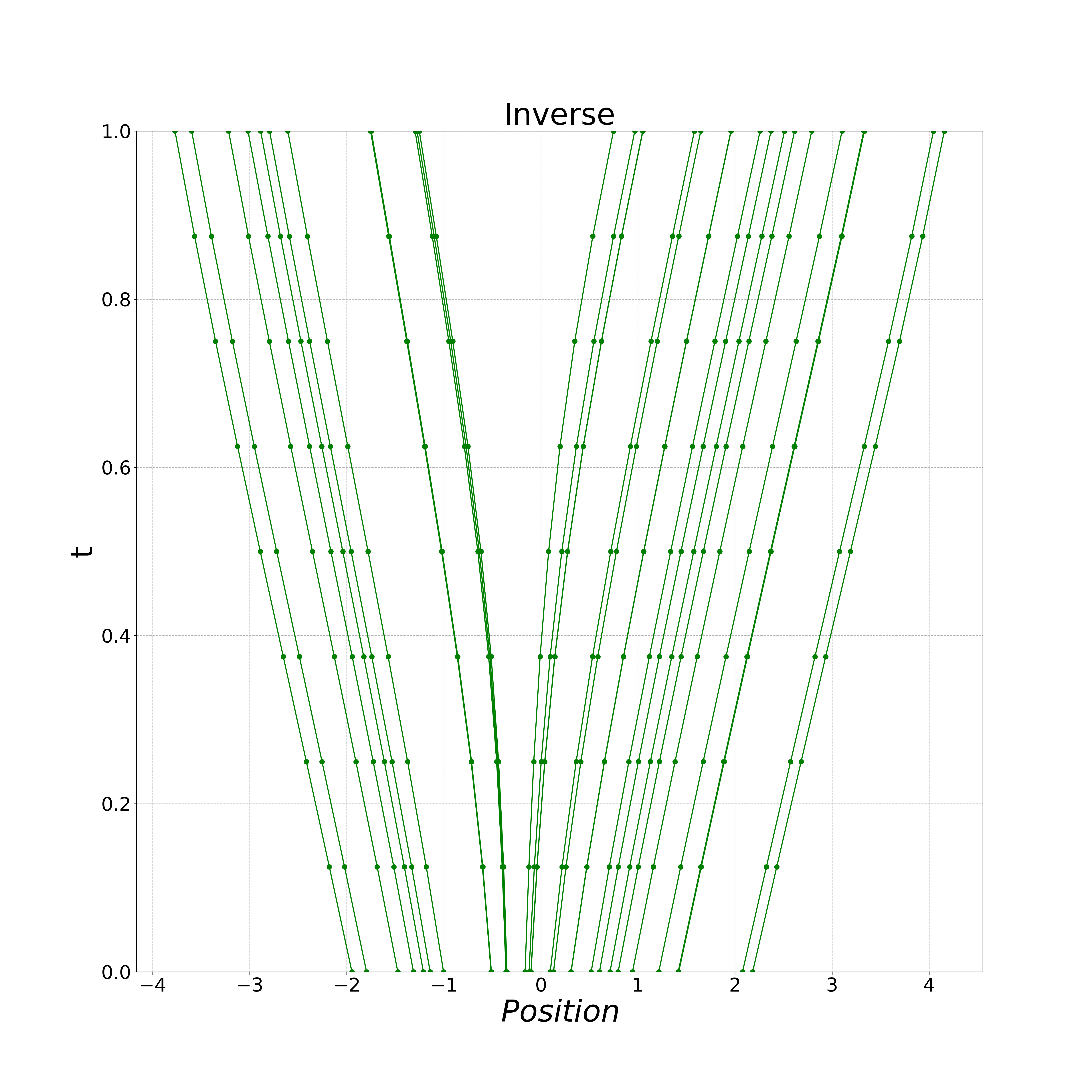}

\end{minipage}
}
\centering
\caption{Trajectories of 30 random particles in our framework. The x-axis indicates the position of particles and y-axis indicates the time interval. (a) initial positions are drawn from $\rho_{0}$. (b) initial positions are drawn from $\rho_{1}$ }

\label{fig:forward_path}
\end{figure}


Then, we plot the weight change over time with different $\alpha$ in forward and inverse transformation respectively in Figure \ref{fig:weight}. The different weight changing trends between particles verify the property of unbalanced optimal transport. Especially the particles in the region corresponding to high target density have greater weight, in the sense that the weight change in UOT can play similar role of transport in the whole transformation. We also observe that as $\alpha$ decreases, the mass creation/destruction effect becomes more evident which agrees with \eqref{eq:weight_PDE}. To see the particle trajectories more clearly, we randomly choose each 30 particles from $\rho_0$ and $\rho_{1}$ to plot their trajectories for forward and inverse transform in Figure \ref{fig:forward_path}. The slope of trajectories can roughly indicate the velocity change, which is consistent with the relationship between particle velocity and weight \eqref{eq:veolocity and weight}. Particles with increasing weight slow down when being about to reach their destinations. The velocity of particles with decreasing weight increases, which indicates that they tend to leave present location.

Summarily, unbalanced optimal transport introduce a more general particle transformation involved both location and weight. The coefficient $\alpha$ controls the effect of mass creation/destruction, by adjusting which we can decide how weight change influences the whole transformation.

\subsection*{The effects of $\gamma_1$ and $\alpha$ on $J_{\mathrm{SWFR}}$}

In this part we set $\gamma_2=0$ for convenience to see how $J_{\mathrm{SWFR}}$ varies against different $\gamma_1$ and $\alpha$. 

We fix $\alpha=1$ and draw 2048 random particles following the distribution $\rho_0$ as training set. Figure \ref{fig:compare_alpha_gamma} (a) shows $J_{\mathrm{SWFR}}$ against different $\gamma_1$. We observe that as $\gamma_1$ goes to zero, $J_{\mathrm{SWFR}}$ increases and converges to some limit which agrees with Proposition \ref{prop:small_gamma}.

Furthermore, Proposition \ref{prop:small_gamma} suggests that if we take a sufficiently small $\gamma_1$, the model can give a good approximation to the geodesics under spherical WFR metric. In this small $\gamma_1$ regime, as training process goes on, one can expect $\mathrm{D}_{\mathrm{KL}}[\rho_0(x)\| \widetilde{\rho}(x,T)] \approx 0$ and $J_{\mathrm{SWFR}}\approx \mathrm{d}_{\mathrm{SWFR},\alpha}^2(\rho_0, \rho_1)$. In Figure \ref{fig:compare_alpha_gamma} (b), we compare $J_{\mathrm{SWFR}}$ against different $\alpha$ when training reaches the equilibrium state. The experiment is repeated independently for $5$ times. We can find that as $\alpha$ increases, $J_{\mathrm{SWFR}}$ has a slower increasing rate, which is expected to converge to $W_2^2(\rho_0,\rho_1)$ according to \eqref{eq:weight_PDE}. To make a comparison, we use the same training set to compute $\mathbb{E}_{\rho_0} L(x,T)$ in \eqref{eq:OT-Flow} for OT-Flow model, which turns out to be 2.71 and larger than $J_{\mathrm{SWFR}}$ in our framework.

\begin{figure}[htbp]
\centering
\vspace{-0.5cm}
\subfigure[$J_{\mathrm{SWFR}}$ versus $\gamma_1$]{
\begin{minipage}[t]{0.45\linewidth}
\centering
\includegraphics[width=6cm]{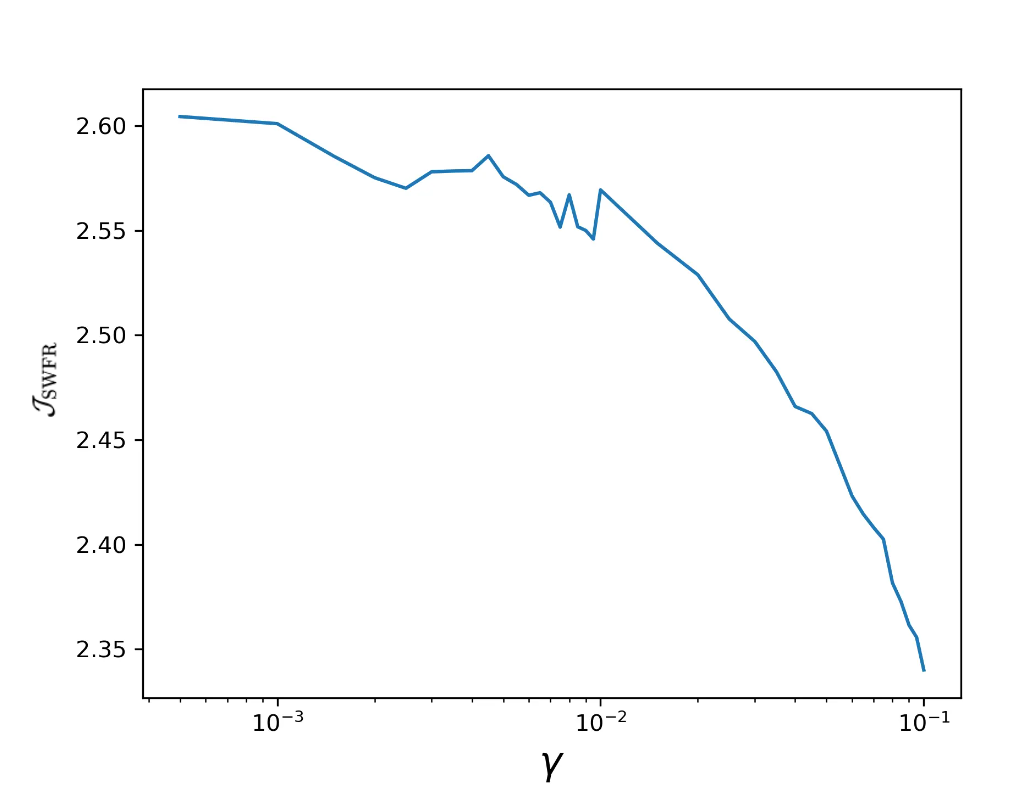}
\end{minipage}
}
\subfigure[$J_{\mathrm{SWFR}}$ versus $\alpha$]{
\begin{minipage}[t]{0.45\linewidth}
\centering
\includegraphics[width=6cm]{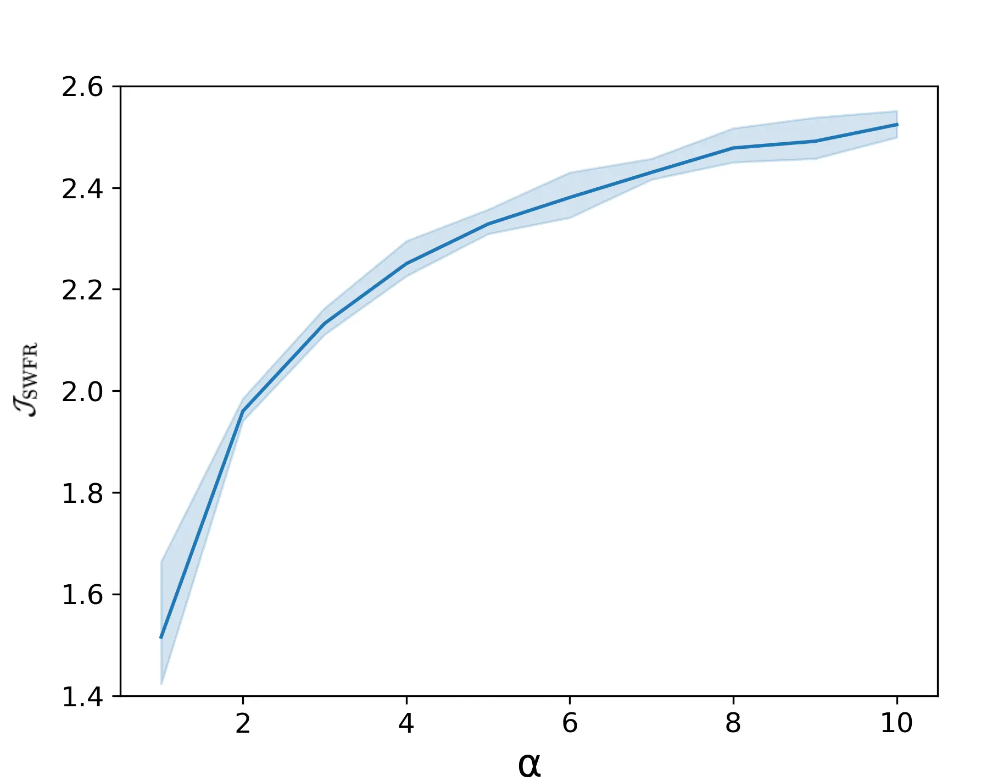}

\end{minipage}
}
\centering
\caption{$J_{\mathrm{SWFR}}$ (squared spherical WFR distance) with different $\gamma_1$ or
$\alpha$}
\label{fig:compare_alpha_gamma}
\end{figure}

\subsection{2-D toy problems} 

In this section, we demonstrate the accuracy and generation quality of our model on two 2-D toy problems also used in \cite{onken2020ot}. We use trained $\Phi$ to simulate an inverse transformation (sampling from 2-D standard normal distribution and pushing particles back to the data distribution), from which one can compare the similarity of original data with generative distribution. The high similarity indicates that our model can generate weighted samples to approximate $\rho_0$ with satisfactory accuracy even though $\rho_0$ has separate supports. Numerical results in type of distribution heatmaps are illustrated in Figure \ref{fig:2d}. The transformation we mentioned refers to considering both location and weight change of particles.

\begin{figure}[htbp]
\centering
\subfigure[8 Gaussians]{
\begin{minipage}[t]{0.48\linewidth}
\centering
\includegraphics[width=5.5cm]{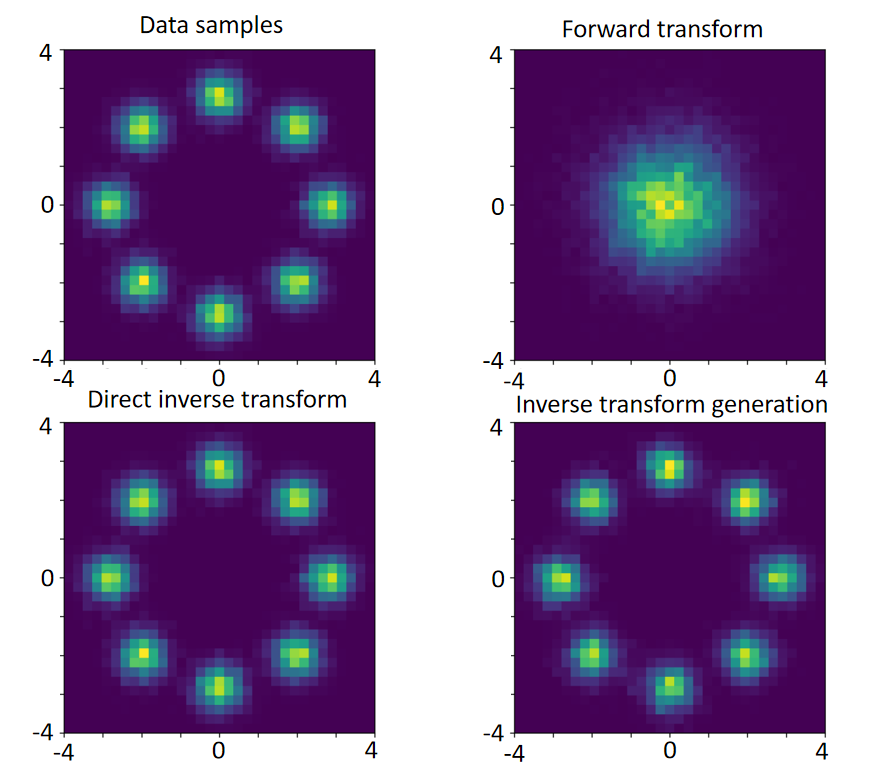}
\end{minipage}
}
\subfigure[Moons]{
\begin{minipage}[t]{0.48\linewidth}
\centering
\includegraphics[width=5.5cm]{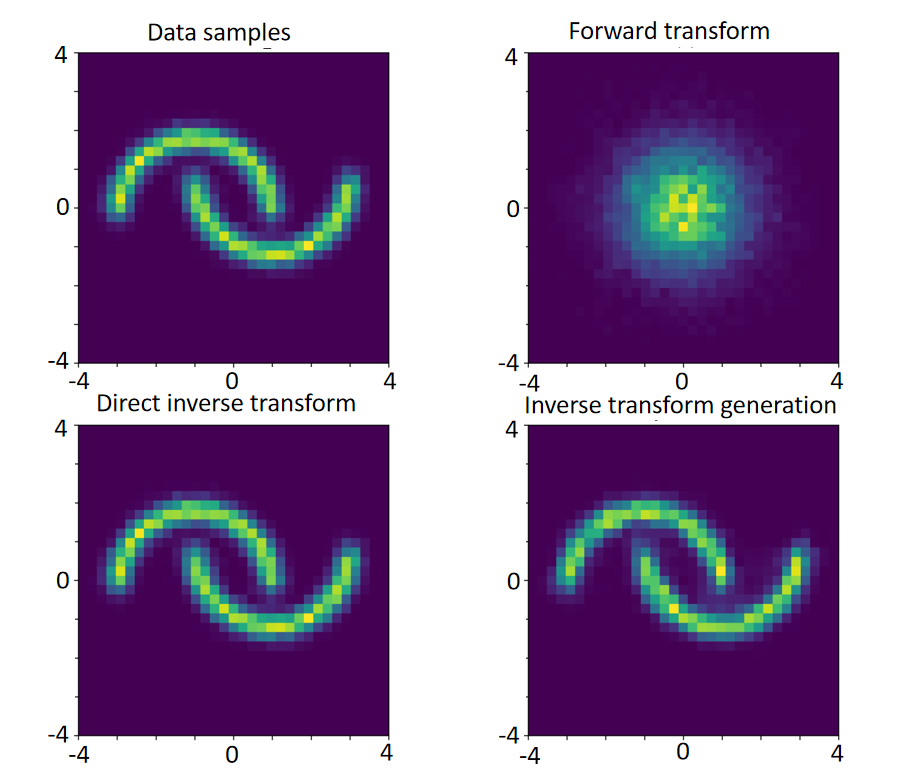}
\end{minipage}%
}%

\centering
\caption{Performance heatmaps for 2-D toy problems. In each example: (1) top left: samples from the unknown distribution $\rho_0$ (2) top right: the forward transformation of particles from $\rho_{0}$ (3) bottom left: direct inverse transformation of weighted particles in top right figure. (4) bottom right: samples generated by inverse transformation from standard normal distribution.
}
\label{fig:2d}
\end{figure}

\subsection{Generating Weighted Samples from Bayesian Posterior}





We apply our UOT-gen model to a Bayesian problem \cite{wu2022ensemble}. Consider Bernoulli equation:
 
\begin{equation}\label{Bernoulli}
\frac{d v}{d t}-v=-v^{3}, \quad v(0)=x
\end{equation}
whose analytic solution is given as follows:
\begin{equation}
v(t)=G(x, t)=x\left(x^{2}+\left(1-x^{2}\right) e^{-2 t}\right)^{-1 / 2}.
\end{equation}
Suppose we can collect the observations of $v(t)$ at different time $t=n \Delta t, n=1,2,\cdots,T$. Assume the observation will introduce a zero-mean Gaussian noise with standard deviation $\sigma$. We are supposed to estimate the initial position from the sequential observed data. This model is a typical problem for data assimilation methods since it exhibits certain non-Gaussian behavior \cite{apte2007sampling}. In our experiment we set $T=50, \Delta t =1, \sigma=0.4$. The ground truth initial position is $x=0.2$. First we show a sequence of observed data and analytic solution in Fig \ref{Fig.main1}.



\begin{figure}[htbp]
\centering
\begin{minipage}[t]{0.45\textwidth}
\centering
\includegraphics[width=6cm]{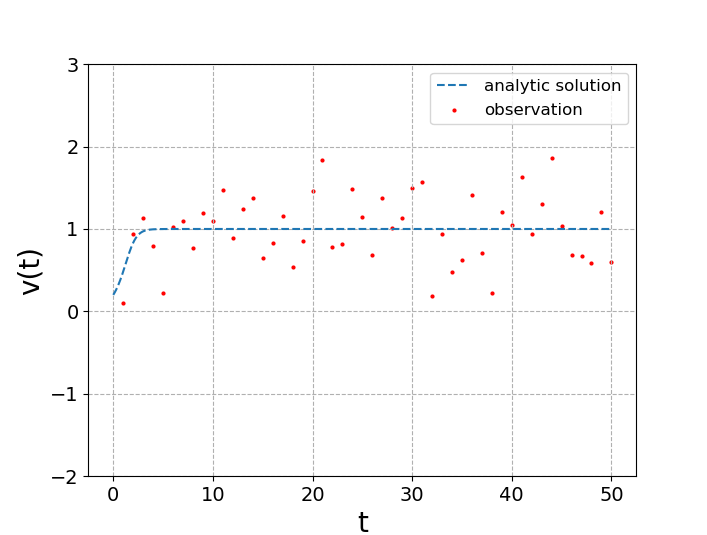}
\caption{Analytic solution and simulated observation for $\sigma=0.4$.}
\label{Fig.main1}
\end{minipage}
\hspace{10mm}
\begin{minipage}[t]{0.45\textwidth}
\centering
\includegraphics[width=6cm]{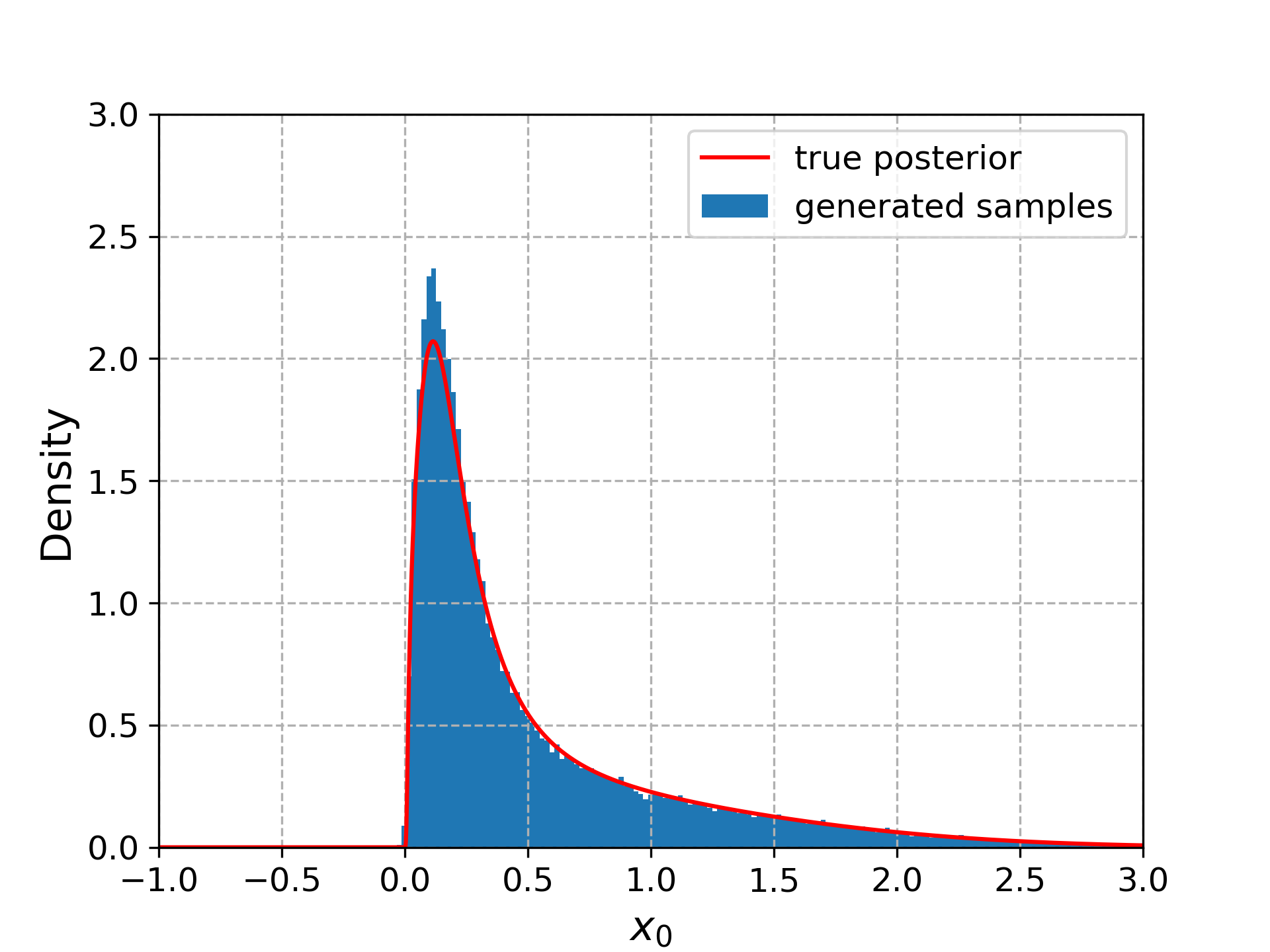}
\caption{Comparison of histgram of generated samples and true posterior.}
\label{Fig.main2}
\end{minipage}
\end{figure}

We can use Bayesian formula to get weighted samples satisfying posterior. Without loss of generality, we set prior $p(w)$ as $\mathcal{N}(0.5,1)$. Then we draw samples from prior and calculate corresponding likelihood $p(\mathcal{D}|w)$, which will be normalized as weights. $D_{t}$ denotes the observation data at time $t$, then the exact $p(\mathcal{D}|w)$ is 

\begin{equation}
    \begin{aligned}
        p(\mathcal{D}|x_{0})=\prod_{t=1}^{50}\frac{1}{\sqrt{2 \pi \sigma^{2}}} \exp\left(-\frac{(D_{t}-G(x_{0},t))^{2}}{2\sigma^{2}}\right).
    \end{aligned}
\end{equation}
Then posterior is
\begin{equation}
    \begin{aligned}
        p(x_{0}|\mathcal{D})=\frac{1}{p(\mathcal{D})}\frac{1}{\sqrt{2 \pi }} \exp\left(-\frac{(x_{0}-0.5)^{2}}{2}\right)\prod_{t=1}^{50}\frac{1}{\sqrt{2 \pi \sigma^{2}}} \exp\left(-\frac{(D_{t}-G(x_{0},t))^{2}}{2\sigma^{2}}\right).
    \end{aligned}
\end{equation}

We plot the unnormalized posterior for the convenience for comparison with weighted samples generated by our UOT-gen model later. Our UOT-gen model is supposed to use given weighted samples to generative weighted sample satisfying posterior. We draw 2048 samples from the prior in our experiment. We set above 2048 samples with corresponding weights as discrete $\rho_{0}$, $\mathcal{N}(0,1)$ as $\rho_{1}$ to train our UOT-gen model. Fig \ref{Fig.main2} shows the comparison of the new generated samples with true posterior. The histgram of new generated samples fit posterior well.


We also apply the online algorithm \ref{alg:online} on the above Bernoulli example. We set $\Delta t=0.1, m=5$ with other settings unchanged, which means we collected observations in each $5\Delta t$ and then update UOT-gen iteratively. Fig \ref{Fig.main5} shows the comparison of true posterior with generated samples from updated UOT-gen in four different time points. The well-fitted density estimations indicate that our online algorithm can generate weighted samples satisfying updated posterior after obtaining new observations.

\begin{figure}[H] 
\centering 
\includegraphics[width=1\linewidth]{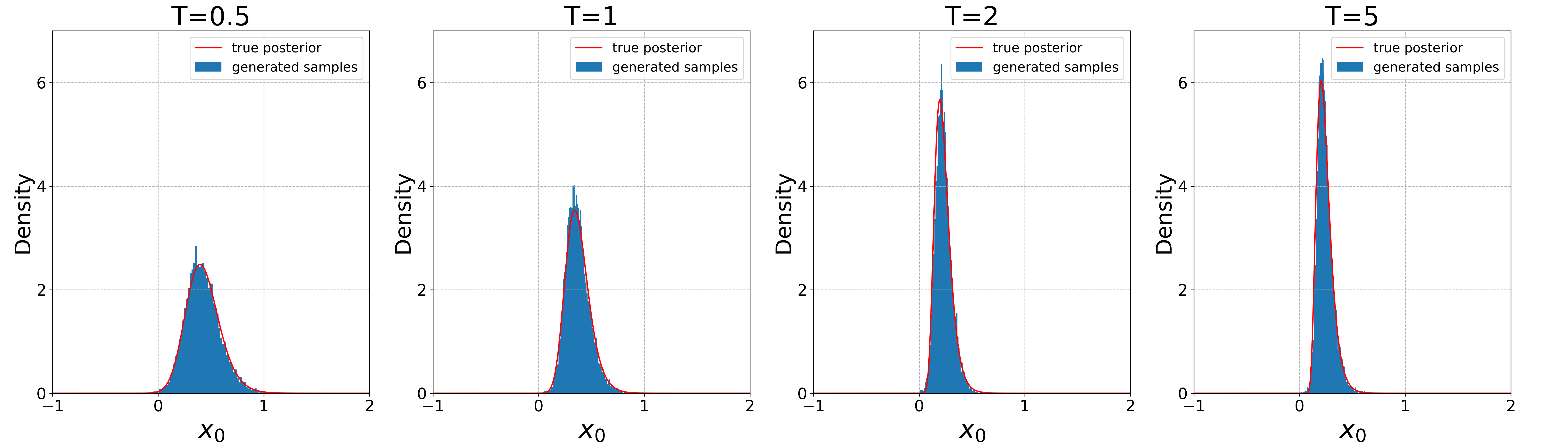} 
\caption{Using online algorithm to generate weighted samples for Bernoulli example above.} 
\label{Fig.main5} 
\end{figure}

\section{Conclusion and discussion}\label{sec:dis}

In summary, we have developed a deep learning framework to compute the geodesics under the spherical WFR metric based on a Benamou-Brenier type dynamic formulation. A KL divergence term based on inverse mapping was introduced into cost function as a soft boundary constraint to tackle the problem arising from weight change. Also, we leveraged the relationship between particle velocity and weight to introduce a new regularization term into our model. Then we demonstrated that the learned geodesics can be used to generate weighted samples from some target distribution. Numerical results have shown the accuracy and efficiency of our model, especially beneficial for applications with given weighted samples in Bayesian inference.

Our framework is promising in dealing with weighted data, which can be costly or even infeasible for previous flow models. Future topics include applying the UOT-gen model in other fields with weighted samples, such as general Bayesian inference tasks and new drug molecule design, where one may hope to keep the most already known structures.

\section{Acknowledgement}
This work is partially supported by the National Key R\&D Program of China No. 2020YFA0712000 and No. 2021YFA1002800. The work of L. Li was partially supported by Shanghai Municipal Science and Technology Major Project 2021SHZDZX0102, NSFC 11901389 and 12031013, and Shanghai Science and Technology Commission Grant No. 21JC1402900.


\normalem
\bibliographystyle{ieeetr}
\bibliography{WFR}

\end{document}